\documentclass{article} 
\usepackage[preprint]{colm2026_conference}
\usepackage{lineno}

\usepackage[utf8]{inputenc} 
\usepackage[T1]{fontenc}    
\usepackage{hyperref}       
\usepackage{url}            
\usepackage{booktabs}       
\usepackage{amsfonts}       
\usepackage{amsthm} 
\usepackage{nicefrac}       
\usepackage{microtype}      
\usepackage{xcolor}         
\usepackage{enumitem}
\usepackage{natbib}
\usepackage[table, dvipsnames]{xcolor}

\usepackage{amsmath}
\usepackage{forest}
\usepackage{tikz}
\usetikzlibrary{decorations.pathreplacing, fit, backgrounds, calc}
\usepackage{adjustbox}

\usepackage[nameinlink,capitalize]{cleveref}
\definecolor{darkblue}{rgb}{0, 0, 0.5}
\hypersetup{colorlinks=true, citecolor=darkblue, linkcolor=darkblue, urlcolor=darkblue}

\newcommand{\lnorm}[1]{{\textsf{N}}(#1)}

\newcommand{\ninf}[1]
{{\left\|#1\right\|_{\max}}}

\usepackage[mathscr]{euscript}

\newcommand{\sft}[1]{\mathbb{S}\left(#1\right)}
\newcommand{\sfti}[2]{\mathbb{S}_{#1}(#2)}

\usepackage{mathtools}

\definecolor{darkred}{RGB}{150,0,0}
\definecolor{darkgreen}{RGB}{0,150,0}
\definecolor{darkblue}{RGB}{0,0,200}

\newtheorem{theorem}{Theorem}
\newtheorem{lemma}{Lemma}

\newtheorem{remark}{Remark}

\newcommand{\appropto}{\mathrel{\vcenter{
  \offinterlineskip\halign{\hfil$##$\cr
    \propto\cr\noalign{\kern2pt}\sim\cr\noalign{\kern-2pt}}}}}
    
\newcommand{\cut}[1]{\textcolor{red}{}}

\newcommand{\Lc}{\mathcal{L}}

\newcommand{\beq}{\begin{equation}}
\newcommand{\eeq}{\end{equation}}
\newcommand{\bea}{\begin{align}}
\newcommand{\eea}{\end{align}}

\newcommand{\R}{\mathbb{R}}

  \newcommand{\eps}{\epsilon}

\DeclarePairedDelimiterX{\inp}[2]{\langle}{\rangle}{#1, #2}

\title{Structure Before Collapse: Transient semantic geometry in next-token prediction}

\author{Yize Zhao\textsuperscript{1}, Isabel Papadimitriou\textsuperscript{2}*  \& Christos Thrampoulidis\textsuperscript{1}*
\\
\textsuperscript{1}Department of Electrical and Computer Engineering, \textsuperscript{2}Department of Linguistics\\
The University of British Columbia\\
\texttt{\{zhaoyize, cthrampo\}@ece.ubc.ca}\\
\texttt{isabel.papadimitriou@ubc.ca}
}

\newif\ifshownotes
\shownotestrue 

\begin{document}

\ifcolmsubmission
\linenumbers
\fi

\maketitle
{\renewcommand{\thefootnote}{}\footnotetext{$^\ast$Equal senior authorship.}}

\begin{abstract}
Neural Collapse predicts that balanced one-hot classification pushes model representations to be equally far from each other; a symmetric configuration that depends only on the output label and ignores any semantic similarity in the inputs. This creates a puzzle: next-token prediction language models are trained predominantly (as context length increases) with one-hot labels: the same context is very unlikely to appear twice in training with different labels. However, they clearly learn latent structural features. That is, despite the one-hot training regime, a language model's contextual embeddings represent the fact that the next word in ``Mary broke the \underline{\hspace{0.8
cm}}'' is likely to be filled by tokens in the latent classes of a) medium-sized, b) rigid, c) inanimate nouns. 
How does gradient descent find such categorical semantic structure when co-occurrence statistics collapse to one-hot sparsity, eliminating any shared next-tokens among different contexts? To investigate this tension we identify three synthetic controlled settings where inputs have latent semantic factors but are mapped to distinct one-hot labels. We find that semantic geometry emerges early in training, and that representations cluster by shared attributes despite receiving no explicit supervision to do so. This structure is transient: with sufficient capacity and time, the model eventually reaches the predicted symmetric state where all representations are equally separated. We study this phase transition through Gram matrix analysis and propose a preliminary modification to the commonly used unconstrained features model to capture the emergent semantic geometry.
\end{abstract}

\section{Introduction}
How do language models learn latent structure in language when, overwhelmingly, the long training contexts that they use to learn next-word-prediction are unique \citep{goodman2001bit, buck2014n, lu2025salieri}? Acquiring linguistic capabilities involves deducing the multiple levels of latent structure in linguistic input, from semantic similarity relationships to abstract syntactic structure to the subtle rules around rare, semi-fossilized constructions. 
Though in practice language models encode and use
much of this latent information,
this is in fact in tension with current theory.  For cases of balanced one-hot supervision, (where each context is only paired with one possible label in the training data) the literature on Neural Collapse \citep[$\mathcal{NC}$,][]{papyan2020prevalence} predicts that learnt representations should form a symmetric Simplex Equiangular Tight Frame (ETF) in which representations are
maximally separated and dictated solely by the one-hot labels and retain no trace of the underlying input semantics.
\footnote{Throughout, we use ``semantics'' as a shorthand to mean structured relationships of similarity between contexts, but we do not take a theoretical stance on whether the structures in our synthetic empirical setups solely reflect linguistic semantics, or if they also apply to syntactic and other relationships.  We discuss this in more detail in \hyperref[sec:exp3]{Experiment 3}.} 


{\emph{Where, then, does this semantic structure come from?} One classical answer
locates it in the distribution of co-occurring
labels: contexts with similar meanings tend to be followed by similar
distributions of words. A long line of distributional methods, from
$n$-gram models to word2vec, recovers semantic relationships by exploiting
exactly this label co-occurence signal \citep{mikolov2013distributed,levy,pennington2014glove}.
Recent theoretical work makes an analogous mechanism precise for next-token
prediction:  When a context is paired not with a single label but with a
sparse soft distribution over labels (e.g.,  ``San Francisco is famous for its'' followed by ``skyscrapers'', ``beaches'', and ``restaurants'', but not by ``sun''),
 shared labels act as a bridge that aligns the representations of distinct contexts, yielding
a meaningful and semantically rich geometry  \citep{zhao2024implicit, zhao2025geometry}.  
However, the probability of encountering these bridges through repeated contexts
vanishes as context length grows. The majority of next-token supervision in
modern long-context training therefore falls into the one-hot regime.} This brings us back to the central questions motivating our experiments: 

\begin{center} \emph{How can Gradient Descent find structured semantic geometry when the data is strictly one-hot? \\
\vspace{0.5ex} If the terminal state of optimization is a symmetric ETF that erases latent relationships {between labels}, under what conditions does such a  semantic geometry emerge?} \end{center}

In this paper, we aim to draw an empirical and theoretical bridge between the theory on Neural Collapse and the realities of language model training. We do this by proposing a set of controlled synthetic toy languages with known latent structuring similarities, and train a suite of transformer models to model them, observing the representations formed during training. For our toy languages, we can accurately describe a representational structure (in terms of a representational similarity, or, Gram matrix) that leverages the semantic similarities, and contrast it with the ETF representational structure that's predicted by the NC theory, that relies only on the one-hot label and not the structure of the context. This way, we can accurately track the training dynamics of structured semantic representation.

Our main finding is that, in the synthetic language learning experiments that we run, \textbf{Gradient Descent recovers structured semantic geometry {organized by shared latent factors} early in training} (despite purely one-hot next-token supervision) before collapsing to a terminal ETF geometry {organized only by output labels}. 
Our contributions are summarized as follows:

\begin{enumerate}[leftmargin=*]
    \item We identify a previously unexamined contradiction: Neural Collapse theory predicts that one-hot supervision erases latent semantic structure from representations, yet language models trained under effectively one-hot regimes demonstrably learn such structure.
    
    \item We introduce an empirical paradigm 
    for tracking whether learned representations reflect latent semantic structure or collapse to the ETF geometry predicted by $\mathcal{NC}$: Representational Similarity Analysis (RSA) to latent semantic categories in synthetic languages. 
    
    \item We identify three synthetic languages where optimization recovers structured, semantic geometry, despite purely one-hot next-token supervision, and show that a) this semantic geometry is a transient phase occurring early in training, and  b) given sufficient model capacity and training time, it eventually yields its place to a symmetric ETF, reconciling our findings with standard $\mathcal{NC}$ theory. 



    \item We present a simplified mathematical model that empirically and theoretically successfully reproduces this transient semantic alignment.

\end{enumerate}

\begin{figure}[t]
    \centering
    \vspace{-1.1cm}
    \makebox[\linewidth][c]{%
        \begin{minipage}{1.2\linewidth} 
            \includegraphics[width=\linewidth]{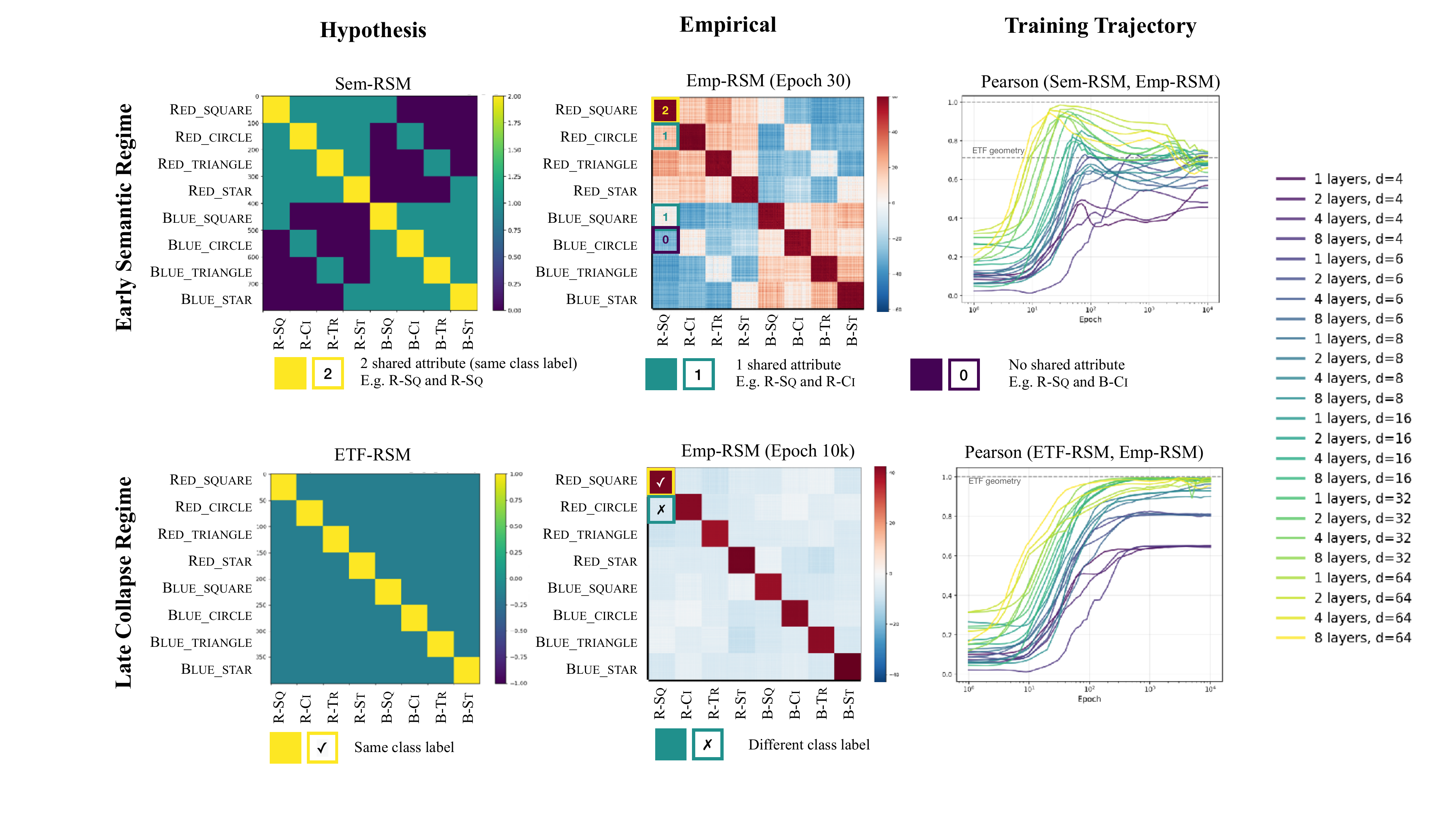} 
        \end{minipage}
    }
    \vspace{-0.2in}
\caption{
\textbf{Transient semantic geometry emerges before collapse} (Experiment~1). Sample-level $n\times n$ RSMs over the full training set, ordered by output label ($n=800$, $K=8$).
\textbf{Left column:} hypothesis RSMs. Sem-RSM encodes latent semantic overlap between class pairs, with values $2$, $1$, and $0$ equal to the number of shared attributes. ETF-RSM encodes the label-driven collapsed geometry predicted by Neural Collapse.
\textbf{Middle column:} empirical RSMs (Emp-RSM) from the largest model (8 layers, $d=64$) at epochs 30 and 10k. Early in training, the empirical geometry reflects the semantic hierarchy in Sem-RSM; by the end of training, this graded structure weakens and Emp-RSM moves toward the binary same-label versus different-label pattern of ETF-RSM.
\textbf{Right column:} Pearson correlation between Emp-RSM and each hypothesis RSM over training, across model depths and widths. Correlation with Sem-RSM rises early and then declines, while correlation with ETF-RSM increases later, especially for larger models, indicating a transition from an early semantic regime to a late collapse regime.}
        \vspace{-0.2in}
            \label{fig:main_results}
\end{figure}

\vspace{-0.15in}
\section{Background and related work}
\vspace{-0.05in}

\subsection{Neural Collapse geometries}

We use the term NC geometry to refer to the geometric  arrangement of last-layer representations of a deep-net in the Terminal Phase of Training, where training loss is driven to zero. Under class-balanced one-hot labeled data in $K$ balanced classes, \citet{papyan2020prevalence} found empirically that cross-entropy training yields a  highly-symmetric NC geometry: a simplex equiangular tight frame (ETF) Gram matrix:  
$G_\text{ETF}:=(I_K-\frac{1}{K}1_K1_K^\top)\otimes 1_{n/K}1_{n/K}^\top$. 
Embeddings of examples belonging to the \emph{same} class collapse to the same representation, and embeddings of examples belonging to \emph{different} classes are equally (in fact, maximally) separated from each other: representations of images of ``cars'' are equally separated from representations of images of ``trucks'' as of ``dogs''.

Two theoretical directions have shown the $\mathcal{NC}$ geometry to have more structure. Firstly, when classes are imbalanced, the geometry converges to a generalization of ETF, the SELI geometry, where similarity relations are structured according to output label frequency similarities \citep{thrampoulidis2022imbalance}. Secondly, in cases of sparse soft-labeled (rather than one-hot) data,  the resulting geometry is the spectral components of a context-word co-occurrence matrix \citep{zhao2024implicit,zhao2025geometry}. In both cases, however, the \textit{sole driver of the similarity structure} continues to be the output labels. A key factor in this limitation is the unconstrained features model \citep[UFM,][]{mixon2020neural,fang2021exploring,zhu2021geometric}, which is predominantly employed in $\mathcal{NC}$ theory and drives the above geometry characterizations but ignores input-data structure by design.
Prior work has empirically examined the role of input semantics in $\mathcal{NC}$ in the context of one-hot image classification \citep{yang2023neurons, jiang2024generalized}, showing that in various empirical settings, deep-net classifiers trained on CIFAR labels (or, artificial super-class labels) lead to semantic representations reflecting class relations, seemingly contradicting the ETF theory.

Although we study language data, our findings shed light on the questions left open by these prior works: semantic geometry does appear in the representations of deep-nets, but as a transient phase early in training, before eventually collapsing to ETF.

\vspace{-0.1in}
\subsection{Using synthetic languages to understand language model learning}
\vspace{-0.1in}
Our methodology builds on a line of inquiry in language model interpretability of training models on carefully controlled synthetic data in order to understand their language learning capabilities. Representational analyses on models trained on Markov chains \citep{shai2024belief}, hierarchical grammars \citep{murtycharacterizing, AllenZhu2023PhysicsOL, schulz2025unraveling, cagnetta2024towards}, language-like templates \citep{ahuja2025learning, mccoy2020does, hupkes2020compositionality, petty2021transformers}, and simple synthetic worlds \citep{li2021implicit, li2023emergent, lampinen2025generalization} show how models form latent generalizations that correspond to the data-generating process, while work that investigates how these latent abstractions connect to natural language trains models on both synthetic and natural languages in controlled contexts \citep{papadimitriou-jurafsky-2020-learning, papadimitriou2023injecting, hu2025between, lee2026training}. To understand more linguistically-grounded aspects of learning, \citet{kallini2024mission}, \citet{misra2024language}, and \citet{patil2024fict} train models on synthetically altered versions of natural language, to pinpoint specific learning effects. Recent papers have also used synthetic languages to understand the training dynamics of generalization \citep{qin2024sometimes, zucchet2025language, wang2024rich, zhao2026shattered, gu2025mixing, behnia2025facts, vasudeva2026understanding}. 
These papers establish that trained language models often contain geometrically organized semantic representations. We ask a complementary question: when such geometry appears during training, and how it evolves under strictly one-hot next-token supervision.

\vspace{-0.1in}
\subsection{Representation geometry in language models}
\vspace{-0.1in}

{That semantic and syntactic structure appears geometrically in language representations
was established early in distributional word representations, where such regularities in embedding space are driven by word co-occurrence statistics
\citep{mikolov2013linguistic,pennington2014glove,levy,arora2016latent,levy_semantics}; probing and visualization studies carry this into contextual representations
\citep{hewitt2019structural,tenney2019bert,reif2019visualizing}, and more recent work
finds semantic information, latent task or world states, and interpretable features
organized geometrically in trained models
\citep{li2021implicit,nanda2023emergent,shai2024belief,gurnee2024language,marks2024geometry,wu2025semantic,cunningham2023sparse,elhage2022toy,park2024linear,park2025categorical,jiang2024origins,engels2024not}. Throughout this tradition, semantic geometry is tied to token co-occurrence statistics.
The line of work most directly connected to ours brings this connection of semantic geometry to co-occurrence statistics into the
next-token-prediction setting \cite{ntp,zhao2024implicit}, where the soft distribution over next tokens acts as a
bridge aligning the representations of distinct contexts into a semantic geometry
\citep{zhao2025geometry,yao2025probability,karkada2026closed,karkada2026symmetry,noroozizadeh2025geometric,behnia2026cross}.
We instead ask here \emph{when} semantic geometry
emerges without it, and how it evolves relative to the label-driven geometry of neural collapse.
}

{Beyond the geometry of trained models, a related line of work asks how representations evolve during training. Studies have examined how the representation geometry changes during different phases in training \citep{li2025tracing}, and more specifically how linguistic features emerge, are expressed in representation and weight space, and connect to overall model loss \citep{saphra2019understanding, chen2024sudden, zhao2025geometry}. We study this dynamics of representation geometry during training in controlled one-hot settings, where the latent semantic structure is known and can be directly compared against the collapse geometry predicted by neural collapse.}

\vspace{-0.12in}
\section{Overview of our empirical approach}
\label{sec:empirical-overview}
\vspace{-0.07in}
We develop three controlled settings where diverse input contexts share latent semantic features but are mapped to one-hot targets. All three toy languages  share this latent-semantics-yet-one-hot character, but they get progressively more language-inspired (and therefore complex). We train language models of varying widths and depths on our toy languages, and observe their training trajectories. All of our toy languages are posed to the model as a single prediction per example (predicting the next word of a sequence),  rather than using a full next-token loss over all sequence positions. In this way, we can ensure that the context-label distribution is strictly one-hot, making the one-hot nature of the supervision as explicit as possible. 

\subsection{Training setup and independent variables: Model depth, width, and training time}
Across all experiments, we train standard transformer decoders on the synthetic languages and track embedding geometry throughout training. For our independent variables, we sweep across different model depths, embedding dimensions, and measure our dependent variables along training. (Specific details in \Cref{app:experimental-details}).

\subsection{Dependent variable: Representational similarity to semantic and collapsed geometries}

To answer ``what similarity structure do the context embeddings of a language model have?,'' we measure the representational similarity between how the models encode the inputs, and (a) the latent semantic similarities of data versus (b) the surface-level output label class. 

Representational Similarity Analysis \citep[RSA,][]{kriegeskorte2008representational, nili2014toolbox} is a method comparing two spaces that both have a similarity metric defined.
RSA works by measuring the correlation between pairwise Representational Similarity Matrices (RSMs) for a number of input samples, $n$.\footnote{RSA is often defined in terms of dissimilarity matrices (rather than similarity matrices), but this is equivalent. When our similarity metric is an inner product, the RSM is a Gram matrix.
}
We track the Pearson correlation between the RSM arrived at empirically by our models (Emp-RSM) with two hypothesis RSMs: a semantic RSM that depends on the language's latent semantic classes (Sem-RSM), and an equiangular tight frame RSM that puts maximal distance between all inputs with different labels (ETF-RSM). 

\textbf{Sem-RSM}\; We define a similarity metric over input context strings
based on the number of shared latent classes that a context selects for (eg, if one context selects for small, inanimate nouns and another selects for large, inanimate nouns, they will have a similarity of 1 shared latent class).
Since we perfectly know the data-generating process of our toy languages, we can construct 
an idealized representational similarity matrix where \textit{the only factors} influencing the similarity between two context representations are latent semantic classes. 

\textbf{ETF-RSM}\; We also create an alternative hypothesis RSM: that of the simplex equiangular tight frame geometry (ETF-RSM). In this case, two input contexts have a similarity of 1 if they take the same output label, and a similarity of 0 otherwise, ignoring any latent semantics that contribute to the contexts.\footnote{Though, of course, output labels do contain some level of semantic information, we contrast this to the Sem-RSM which represents the full latent system underlying the data-generating process} In this idealized RSM, the only feature influencing the similarity between two contexts is their output label.\footnote{The ETF geometry assumes balanced output labels, i.e. equal number of input contexts per label. In our experiments with imbalanced output labels, we use the SELI geometry generalization of ETF, though we maintain the naming ETF-RSM since the common feature of both geometries distinguishing them from Sem-RSM, is that they are solely driven by output labels. }


\textbf{Emp-RSM}\; We compare  these two hypothesis RSMs to an empirical RSM of the model context embeddings (Emp-RSM). The RSM of our context embeddings is simply the Gram matrix of embeddings. For each model, at each training epoch, we take the model's last-layer representations of every input in our training data, and construct the Gram matrix of their inner products. The two dependent variables that we measure are the Pearson correlations between each model checkpoint's context representations and the Sem-RSM and ETF-RSM. 

\textbf{Complementary Metrics.}\; 
{In addition to the core metrics listed above, we use the following complementary diagnostics to separate semantic structure from label-driven collapse.}

{\emph{Disentangling semantic and label effects. }To separate semantic alignment from same-label collapse, we use two different-label diagnostics. The first compares whether different-label pairs that share a latent factor are more similar than different-label pairs that share none. The second recomputes semantic RSA after removing all same-label pairs. Complementary to our main results, these checks, that are presented in \Cref{app:cross_label_diagnostics}, confirm that the semantic signal reflects graded structure across labels, rather than only clustering within each exact output label.}

{\emph{NC-style collapse diagnostics. } We provide direct collapse-style measurements in \Cref{app:nc_diagnostics} and \Cref{app:numerical}. First, we report standard NC1 and NC2 diagnostics, measuring within-label variability and convergence of class means toward the ETF geometry \citep{papyan2020prevalence}. Second, we adapt this idea to semantic attributes, measuring whether examples sharing the same latent attribute form tight clusters. The results support the same training picture as RSA: semantic organization is strongest early, while label-driven collapse strengthens later.}

{\emph{Cross-layer geometry. }In \Cref{app:layerwise_geometry}, we repeat the semantic and label-driven measurements across layers, showing that label-driven geometry is generally strongest near the readout, while semantic structure can remain visible in earlier and intermediate layers. A detailed investigation of the geometry evolution across layers is left to future work. }





In the following sections, we describe the three synthetic languages that we construct, and report how  model embedding geometries correlate with Sem/ETF-RSM.

\vspace{-0.1in}
\section{Experiment 1: Color-Shape categories language}
\label{sec:exp1}
\vspace{-0.05in}
\paragraph{Language definition}
Our first synthetic language is a minimal setting where the output labels depend on the cross of two latent semantic factors. Each input is defined by two factors (in our case, color and shape), and each output label reflects only the cross product of the two (eg, \textsc{Red\_Circle} is a one-hot label that only fires for inputs that are in \textit{both} \textsc{Red} and \textsc{Circle}) More formally:

\textit{Latent Factors.}
We define a Color Factor containing attributes $\mathcal{A}_{\mathcal{C}} = \{\textsc{Red}, \textsc{Blue}\}$ and a Shape Factor ($F=\mathcal{H}$) containing attributes $\mathcal{A}_{\mathcal{H}} = \{\textsc{Circle}, \textsc{Square}, \textsc{Triangle}, \textsc{Star}\}$.

\textit{Semantic Class. }We define a Semantic Class $\mathcal{S}_a$ as the set of all samples sharing a specific attribute (like, "Red"), regardless of the other factors. 
These classes are not mutually exclusive: every sample belongs to exactly one Color class and one Shape class simultaneously.

\textit{Label Class.} The optimization objective is defined by the Label Class $\mathcal{Y}$. A label class $y_k$ corresponds to the unique intersection of semantic factors, $y_k = (a_c, a_h)$, resulting in $K = |\mathcal{A}_{\mathcal{C}}| \times |\mathcal{A}_{\mathcal{H}}| = 8 $ mutually exclusive one-hot targets. For example, a sample that is in $\mathcal{S}_{\textsc{Red}}$ and in $\mathcal{S}_{\textsc{Circle}}$ will always get the label $y = \textsc{Red\_Circle}$, which in terms of the label one-hot distribution has no relation to other labels like $\textsc{Red\_Square}$.

\textit{Context Samples.} Each input is a triplet $x = [t_{\text{ID}}, t_{c}, t_{h}]$. To ensure the model doesn't simply memorize, we introduce variability within each class in the form of nuisance IDs ($t_{\text{ID}}$) for each sample.
The contextual tokens $t_c, t_h$ are drawn from sets of fine-grained variants (e.g., when $a=\textsc{Red},$  the variants are $\{\textsc{Dark\_Red, Light\_Red, Medium\_Red}\}$), so the model never observes the coarse latent factors directly.
The input vocabulary ($V=818$) comprises of 800 unique Nuisance IDs and 18 Contextual Tokens (3 fine-grained variants for each of the 2 color and 4 shape attributes). The total dataset consists of $n=800$ samples. Putting this all together, an example of our prediction objective is:
\begin{align*}
    x = [768\quad \textsc{Dark\_Red}\quad \textsc{Small\_Circle}], \qquad y = \textsc{Red\_Circle}\,.
\end{align*}

\paragraph{Results}

We present our results in \Cref{fig:main_results}. We track the models' representational similarity to the two hypothesis RSMs on the left: Sem-RSM and
ETF-RSM.
Our main result is that models transiently reach maximal similarity to semantic representation early in training (right, top), and then collapse away to the ETF geometry, which they approach monotonically throughout training (right, bottom). We also observe that the largest models (yellower lines) are the most effective at finding both the semantic geometry early in training, and the ETF geometry at the terminal phase, so semantic geometry is not a result of representational bottlenecking. 
Surprisingly we find that the RSA values with Sem-RSM for some models are close to 1,
meaning that the \textit{only} representational structure in the model is our hypothesized semantic structure, and that the model is encoding almost no spurious training artifacts.

\begin{figure}[t]

  \centering
  \begin{tikzpicture}[
    box/.style={draw, fill=white, minimum width=0.55cm, minimum height=0.55cm, inner sep=0pt, font=\scriptsize},
    topbrace/.style={decorate, decoration={brace, amplitude=4pt, raise=2pt}, line width=0.5pt},
    botbrace/.style={decorate, decoration={brace, mirror, amplitude=4pt, raise=2pt}, line width=0.5pt},
  ]
    \def\sp{0.65}

    \node[anchor=east, font=\large] at (-0.3, 0) {$V =$};

    \foreach \i in {0,...,19} {
      \node[box] (t\i) at (\i*\sp, 0) {\i};
    }

    \begin{pgfonlayer}{background}
      \node[draw=none, fill={rgb,255:red,197;green,231;blue,216},
            rounded corners=3pt, inner sep=2pt, fit=(t0)(t1)] {};
      \node[draw=none, fill={rgb,255:red,220;green,240;blue,220},
            rounded corners=3pt, inner sep=2pt, fit=(t2)(t3)] {};
      \node[draw=none, fill={rgb,255:red,220;green,232;blue,245},
            rounded corners=3pt, inner sep=2pt, fit=(t4)(t7)] {};
      \node[draw=none, fill={rgb,255:red,255;green,240;blue,220},
            rounded corners=3pt, inner sep=2pt, fit=(t8)(t11)] {};
      \node[draw=none, fill={rgb,255:red,245;green,235;blue,228},
            rounded corners=3pt, inner sep=2pt, fit=(t12)(t19)] {};
    \end{pgfonlayer}

    \draw[line width=0.5pt]
      ([yshift=2pt]t0.north west) |- ([yshift=6pt]$(t0.north west)!0.5!(t1.north east)$) -| ([yshift=2pt]t1.north east);
    \node[above=8pt, font=\scriptsize\bfseries, text={rgb,255:red,0;green,100;blue,0}]
      at ($(t0.north)!0.5!(t1.north)$) {\textsc{Det}};

    \draw[line width=0.5pt]
      ([yshift=2pt]t2.north west) |- ([yshift=6pt]$(t2.north west)!0.5!(t3.north east)$) -| ([yshift=2pt]t3.north east);
    \node[above=8pt, font=\scriptsize\bfseries, text={rgb,255:red,0;green,100;blue,0}]
      at ($(t2.north)!0.5!(t3.north)$) {\textsc{Adj}};

    \draw[line width=0.5pt]
      ([yshift=2pt]t4.north west) |- ([yshift=6pt]$(t4.north west)!0.5!(t7.north east)$) -| ([yshift=2pt]t7.north east);
    \node[above=8pt, font=\scriptsize\bfseries, text={rgb,255:red,0;green,50;blue,120}]
      at ($(t4.north)!0.5!(t7.north)$) {\textsc{Subject}};

    \draw[line width=0.5pt]
      ([yshift=2pt]t8.north west) |- ([yshift=6pt]$(t8.north west)!0.5!(t11.north east)$) -| ([yshift=2pt]t11.north east);
    \node[above=8pt, font=\scriptsize\bfseries, text={rgb,255:red,180;green,100;blue,0}]
      at ($(t8.north)!0.5!(t11.north)$) {\textsc{Verb}};

    \draw[line width=0.5pt]
      ([yshift=2pt]t12.north west) |- ([yshift=6pt]$(t12.north west)!0.5!(t19.north east)$) -| ([yshift=2pt]t19.north east);
 k  \node[above=8pt, font=\scriptsize\bfseries, text={rgb,255:red,140;green,80;blue,60}]
      at ($(t12.north)!0.5!(t19.north)$) {\textsc{Object}};

    \draw[botbrace]
      (t4.south west) -- (t5.south east)
      node[midway, below=8pt, font=\scriptsize, text={rgb,255:red,0;green,50;blue,120}] {$\mathcal{A}_S\!=\!A$};
    \draw[botbrace]
      (t6.south west) -- (t7.south east)
      node[midway, below=8pt, font=\scriptsize, text={rgb,255:red,0;green,50;blue,120}] {$\mathcal{A}_S\!=\!B$};
    \draw[botbrace]
      (t8.south west) -- (t9.south east)
      node[midway, below=8pt, font=\scriptsize, text={rgb,255:red,180;green,100;blue,0}] {$\mathcal{A}_V\!=\!1$};
    \draw[botbrace]
      (t10.south west) -- (t11.south east)
      node[midway, below=8pt, font=\scriptsize, text={rgb,255:red,180;green,100;blue,0}] {$\mathcal{A}_V\!=\!2$};
    \draw[botbrace]
      (t12.south west) -- (t13.south east)
      node[midway, below=8pt, font=\scriptsize, text={rgb,255:red,140;green,80;blue,60}] {$y\!=\!(A,1)$};
    \draw[botbrace]
      (t14.south west) -- (t15.south east)
      node[midway, below=8pt, font=\scriptsize, text={rgb,255:red,140;green,80;blue,60}] {$y\!=\!(A,2)$};
    \draw[botbrace]
      (t16.south west) -- (t17.south east)
      node[midway, below=8pt, font=\scriptsize, text={rgb,255:red,140;green,80;blue,60}] {$y\!=\!(B,1)$};
    \draw[botbrace]
      (t18.south west) -- (t19.south east)
      node[midway, below=8pt, font=\scriptsize, text={rgb,255:red,140;green,80;blue,60}] {$y\!=\!(B,2)$};

  \end{tikzpicture}
  \caption{Linear language from Experiment 2: an illustrative example 20-token vocab with $k=2$ classes, excluding the nuisance IDs. A valid training pair would be $x = [id100 \quad 1 \quad 2 \quad 2 \quad 5 \quad 10 \quad 2]$, \quad $y = 15$. The \textsc{Subject} (5) is in class $A$, and the \textsc{Verb} (10) is in class $2$, so the predicted object would have to be in class $(A, 2)$: $y = 14$ or $y = 15$. In our experiments we use $|V| = 62$ total non-id tokens, and $k=10$ categories for subjects and verbs.}
  \label{fig:diagram_linear_vocab}
  \vspace{-0.1in}
\end{figure}
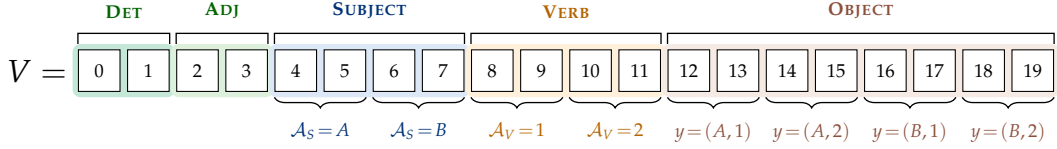

\vspace{-0.1in}
\section{Experiment 2: linear grammar with Zipfian latent categories}
\vspace{-0.1in}

\paragraph{Language definition}
For our second language, we develop the setup of Experiment 1 to be slightly more language-like. We create a synthetic language that follows a simple linear template grammar, and where categories follow a Zipfian distribution. Our vocabulary $V$ is divided into 5 classes, roughly corresponding to a combination of parts of speech and grammatical roles: Determiners, Adjectives, Subjects, Verbs, and Objects, as well as a nuisance ID to make the contexts strictly one-hot.\footnote{Though we use "part of speech" as an intuitive term, the language is purely synthetic: the vocabulary tokens are all integer token ids, and their mapping to PoS classes is arbitrary; see \Cref{fig:diagram_linear_vocab}.} Training samples are drawn from the following linear template, where * indicates that lists of adjectives can be of lengths 0-3:
\begin{align*}
x = [\textsc{Id} \quad \textsc{Det}\quad \textsc{Adj}* \quad\textsc{Subject}_{\mathcal{A}_S} \quad \textsc{Verb}_{\mathcal{A}_V} \quad\textsc{Det} \quad \textsc{Adj}*] \qquad y = \textsc{Object}_{(\mathcal{A}_S \mathcal{A}_V)}\,.
\end{align*}
There are two latent semantic factors: the subclass of the subject $\mathcal{A}_{\mathcal{S}}$ and the subclass of the verb $\mathcal{A}_{\mathcal{V}}$, each with $k$ sub-classes. The prediction label, the \textsc{Object}, has $k\times k = |\mathcal{A}_S \times\mathcal{A}_V|$ sub-classes.
Note there are multiple tokens within each sub-class, so that the actual sub-class remains a latent factor in the input, see \Cref{fig:diagram_linear_vocab}. In order to make our linear language more natural, we sample $\mathcal{A}_V$ from a Zipfian distribution \citep[see][for an analysis of the pervasiveness of Zipfian distributions in natural language]{piantadosi2014zipf}. As such, we use the imbalanced generalization of ETF (the SELI geometry) as our ETF-RSM.



For our design choice of making the object class reliant on a cross-product of the subject class and the verb class, we draw inspiration from how semantic features restrict vocabulary choices in natural language. {Roughly, in a sentence like \textit{The girl ate the \_}, the next word is constrained by both the verb and the subject: it should be an object compatible with eating, and also plausible for the subject. In \textit{The cat devoured the \_}, the verb still selects for an edible object, but the subject changes which objects are most plausible, e.g., ``tuna'' rather than ``cookie''.}

For greater theoretical simplicity, we make the non-realistic assumption that \textsc{Subject}s and \textsc{Object}s are non-overlapping sets of the vocabulary, and that the elements of every \textsc{Object} subclass are also non-overlapping. 

\begin{figure}[t]
    \centering
    \makebox[\linewidth][c]{
    \includegraphics[width=1.1\linewidth]{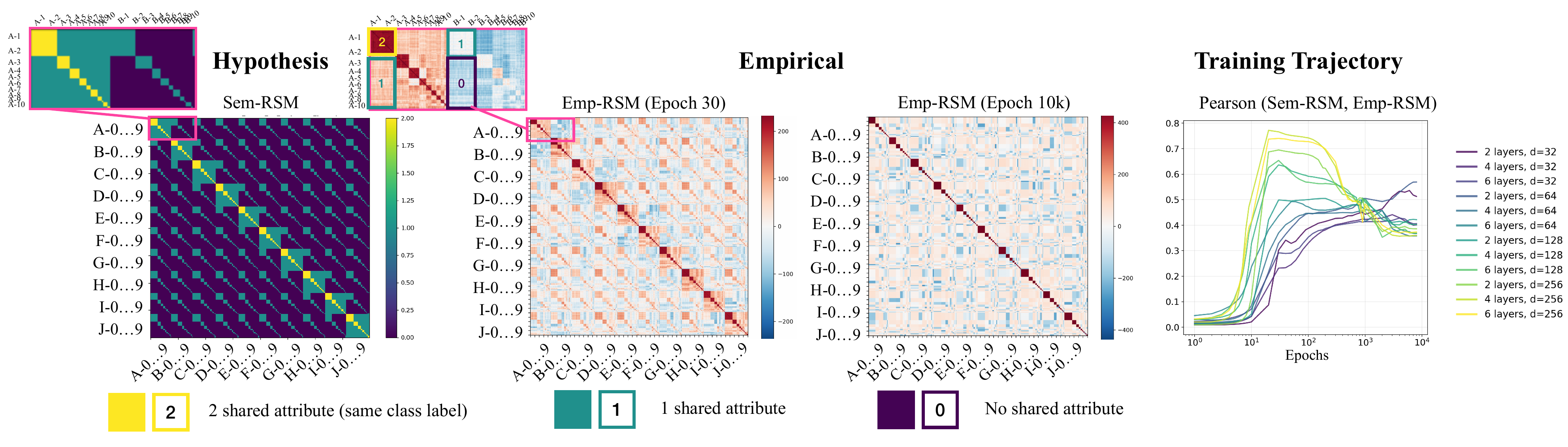}
    }
    \vspace{-0.4cm}
\caption{
\textbf{Transient semantic geometry in the linear template Zipfian language} (Experiment~2). Sample-level $n\times n$ RSMs over the full training set, ordered by semantic and output label ($n=10$k, $K=100$).
\textbf{Left:} Sem-RSM, where similarity is determined by the number of shared latent categories. Because the label distribution is imbalanced, block sizes are unequal (see inset).
\textbf{Middle:} empirical RSMs (Emp-RSM) at epochs 30 and 10k. Early in training, Emp-RSM reflects the off-diagonal semantic structure of Sem-RSM; later, this structure weakens.
\textbf{Right:} Pearson correlation between Emp-RSM and Sem-RSM over training, across model depths and widths. As in Experiment~1, larger models peak higher, and they do so early in training.
}
    \label{fig:language2_result}
    \vspace{-0.3cm}
\end{figure}

\paragraph{Results}
Our results in \Cref{fig:language2_result} show a similar overall pattern to our Experiment 1 results: the Empirical RSMs match the hypothesized semantic RSM (left) early in training, before collapsing away from the structured semantic structure, and the larger-capacity models do this more extremely. Interestingly, at the terminal phase of training, though the representation has moved away from exhibiting the semantic structure, it is not totally collapsed to a diagonal-like label structure: there are still some significant fluctuations in the off-diagonals of the empirical RSM. Though this is not strictly the ETF-RSM's SELI structure (see \Cref{app:additional-exp2}, \Cref{fig:language2_result_full} for the convergence to the ETF-RSM), it is still likely to be an effect of the more unstable learning dynamics of the Zipfian distribution, which we aim to explore in future work. 


\vspace{-0.1in}
\section{Experiment 3: hierarchical grammar with latent categories}
\label{sec:exp3}
\label{sec:lang3}
\vspace{-0.1in}
\paragraph{Language definition}
Our third synthetic language combines a hierarchical grammatical structure with a similar category-selection mechanic as our previous two languages, but where \textit{which} context token it is that defines the category of the label depends on the tree structure (more specifically, the c-command relation). With this language, we aim to capture a controlled toy version of the interactions between grammatical and semantic structure in natural language, where the ways in which meaning restrictions (like tense or negation scope) propagate to future tokens is often dictated by syntactic tree operations.

For our hierarchical grammar
every node in our tree can expand either by coordination ($XP \rightarrow XP\; XP$), right-branching ($XP \rightarrow X\; XP$), or trivial projection ($XP \rightarrow X$). 
Each vocabulary item is also tied to a semantic category, $\mathcal{A}_C$. To sample an $x, y$ pair, we sample a string from the grammar, where the semantic category $\mathcal{A}_C$ is randomly selected for all of the tokens except the last one. The string minus the last token is then our $x$. The last token, $y$, is restricted to be the semantic class of the \textbf{closest token in $x$ that c-commands it}.  C-command is a relation between two nodes in a tree, where $A$ c-commands $B$ if $A$'s direct parent dominates all of $B$'s ancestors (see the red arrow in \Cref{fig:diagram-hierarchical} for an example of c-command), and is an important tree relation in generative linguistics for describing processes like pronoun binding \citep{reinhart1976syntactic, chomsky1988lectures}, negation scope \citep{ladusaw1979polarity}, quantifier scope \citep{may1985logical}, and tense \citep{kratzer1998more}.

To make the model able to induce the tree structure from the language strings, we have two options. The one option would be to have multiple parts of speech and have the model induce the tree structure as in standard grammar induction paradigms. The second option (which we take) is to keep just a single part of speech ($X$), and mark the internal tree structure in the form of parentheses: interleaving in a Dyck language that has the same internal nodes (see the \Cref{fig:diagram-hierarchical} for an example of such a linearization). We elect to have just a single part of speech in order to not conflate the learning of semantic categories with the grammar induction of parts of speech. Having just one part of speech means that the \textit{only relevant categories} for the model to represent are our semantic categories, and so we can measure the RSA to a ground-truth categorical geometric representation.

In future work, we aim to examine the interplay between jointly inducing grammar and learning category-selection semantics. Doing this would require careful and varied analysis of the representation space, as there is no single ground truth similarity matrix. The interface and interference of syntactic and semantic categories is a fundamental topic of analysis across linguistic traditions, and quantifying "ground-truth" representations to represent different hypotheses is a nuanced and multifaceted process. This type of sweep and analysis is beyond the scope of what the current experiment is contributing to this paper's thesis (a hierarchical processing data point), and so we leave it to future work.


\begin{figure}[t]
  \vspace{-1.1cm}
  \centering
  \begin{minipage}[c]{0.35\textwidth}
  \adjustbox{width=0.9\linewidth}{
    \centering
        \begin{tikzpicture}[
      box/.style={draw, fill=white, minimum width=0.6cm, minimum height=0.6cm, inner sep=0pt, font=\footnotesize},
      brace/.style={decorate, decoration={brace, mirror, amplitude=4pt, raise=2pt}, line width=0.7pt},
    ]
      \node[anchor=east, font=\Large] at (-0, -0.75) {$V =\quad$};
 
      \node[box] (topen)  at (0*0.7, 0) {(};
      \node[box] (tclose) at (1*0.7, 0) {)};
      \foreach \i in {0,...,3} {
        \pgfmathtruncatemacro{\col}{\i + 2}
        \node[box] (t\i) at (\col*0.7, 0) {\i};
      }
 
      \foreach \i in {4,...,7} {
        \pgfmathtruncatemacro{\col}{\i - 4}
        \node[box] (t\i) at (\col*0.7 + 1.4, -1.5) {\i};
      }
 
      \begin{pgfonlayer}{background}
        \node[draw=none, fill={rgb,255:red,255;green,242;blue,228},
              rounded corners=4pt, inner sep=3pt,
              fit=(topen)(t3)] {};
        \node[draw=none, fill={rgb,255:red,255;green,242;blue,228},
              rounded corners=4pt, inner sep=3pt,
              fit=(t4)(t7)] {};
      \end{pgfonlayer}
 
      \draw[brace]
        (t0.south west) -- (t1.south east)
        node[midway, below=8pt, font=\footnotesize] {$\mathcal{A}_C\!=\!a$};
      \draw[brace]
        (t2.south west) -- (t3.south east)
        node[midway, below=8pt, font=\footnotesize] {$\mathcal{A}_C\!=\!b$};
      \draw[brace]
        (t4.south west) -- (t5.south east)
        node[midway, below=8pt, font=\footnotesize] {$\mathcal{A}_C\!=\!c$};
      \draw[brace]
        (t6.south west) -- (t7.south east)
        node[midway, below=8pt, font=\footnotesize] {$\mathcal{A}_C\!=\!d$};
    \end{tikzpicture}
    }
  \end{minipage}%
  \hspace{0cm}
  \begin{minipage}[c]{0.35\textwidth}
    \centering
    \adjustbox{max width=0.7\linewidth}{
    \begin{forest}
      for tree={s sep=6mm, inner sep=2pt, before computing xy={l=7mm}}
      [(\quad)
        [(\quad)
          [2]
          [(\quad)
              [6]]]
        [(\quad)
          [4, name=ccom]
          [(\quad)
             [(\quad)
               [(\quad) [1]
                ]
               [(\quad) [3]
                 ]
               ]
              [(\quad) [{$y=5$}, name=y]]]
        ]
      ]
      \draw[->, line width=2pt, red, opacity=0.4, bend left=40] (y) to (ccom);
    \end{forest}
    }
  \end{minipage}
      \caption{An illustrative example of our Exp 3 hierarchical language with $|V|=10, k=2$. To linearize the tree interleaving our probabilistic context free grammar with the Dyck language, we produce an opening or closing bracket whenever an internal node is first and last encountered. So, the training pair for the pictured tree is $x = \;(\; ( \; 2 \; ( \; 6 \; ) \; ) \; ( \; 4 \; ( \; ( \; ( \; (\; 1 \; ) \; ( \; 3 \; )\; ) \; ( \;$, \quad $y= 5$. In this example, $y$ has to belong to $\mathcal{A}_c = c$, as that is the class of the token that c-commands it (4). 
      Trailing closing parentheses are excluded from the prediction objective.  } 
  \label{fig:diagram-hierarchical}
  \vspace{-0.1in}
\end{figure}
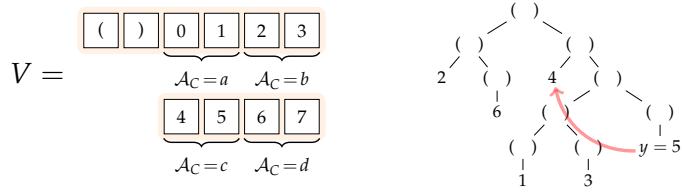

\paragraph{Results}

\Cref{fig:language3_result} shows the same qualitative pattern as in the first two experiments, now in a hierarchical setting. Early in training, the empirical RSM closely matches the semantic hypothesis: examples controlled by the same latent category form clear blocks, and the correlation with Sem-RSM rises sharply to a high peak. This indicates that the model organizes representations by the latent semantic category relevant for prediction, even though that category is mediated by tree structure rather than a fixed linear position. Note that the RSA to the semantic representation peaks later, and the semantic block-diagonals are still faintly present even in the terminal phase of training. This effect is strongest for the smaller-width models ($d=8,16$), which keep higher semantic alignment than the larger-width models. One possible explanation is that the hierarchical structure is both harder to learn and acts as a weak regularizer towards the lower-rank semantic organization, especially when width is limited and the model cannot as easily move to a finer label-level geometry. 
This is an avenue for exploring in future work on linguistic learning dynamics.  

\section{Theoretical proxy: spherical Bag-of-Words model}
\vspace{-0.1in}
We introduce a minimal Spherical Bag-of-Words model that reproduces the same qualitative transient as the transformer experiments. This gives a simpler setting where we can study the mechanism more directly. In this model, representations first align with input-side latent structure, and later move toward label-driven geometry. Thus, the transient does not seem to require self-attention or depth, but can already appear from the interaction between input structure, optimization, and the geometry of the objective.

Our architecture is a \emph{Spherical continuous bag-of-words (BoW)} model that trains two weight matrices $W\in\R^{K\times d}$ and $H\in\R^{d\times V}$ by minimizing the CE loss between input token encoding matrix $X\in\{0,1\}^{V\times n}$ and one-hot label matrix $Y\in\{0,1\}^{K\times n}$: $\Lc(W,H)=\mathrm{KL}\left( Y||\text{Softmax}({W\cdot\lnorm{HX}})\right)$. Here, $K$ is the number of label classes, $V$ is the token vocabulary of the input contexts, and $d$ the hidden dimension, 
the KL operator averages over divergence between columns of the two input matrices, 
and, $\lnorm{\cdot}$ denotes column-wise normalization onto the unit Euclidean sphere. Note this is identical to a continuous BoW \citep{mikolov2013efficient} model except for the normalization, which however we find is crucial to establish terminal convergence to ETF.

We study this Spherical BoW framework on our Experiment 1 language. Firstly, in \Cref{sec:theory_sims,sec:theory_sims_2} we show empirically 
that when trained with Gradient Descent with small random initialization and $d\geq K$, Spherical BoW possesses the desired transient semantic behavior: early on in training, the analytical Emp-RSM $G_\text{emp}:=X^\top H^\top H X$ aligns with the data's Sem-RSM $G_\text{sem}:=B^\top B \otimes 1_m 1_m^\top$ (here, $B$ is the base matrix of shared semantic features; see appendix), before eventually converging to (a scaling of) the ETF-RSM $G_\text{ETF}=(I_K-\frac{1}{K}1_K1_K^\top)\otimes 1_m1_m^\top$. Moreover, the larger the hidden dimension, the more pronounced this behavior, consistent with our earlier experiments.

\begin{figure}
    \centering
    \vspace{-1.3cm}
    \makebox[\linewidth][c]{
    \includegraphics[width=1.1\linewidth]{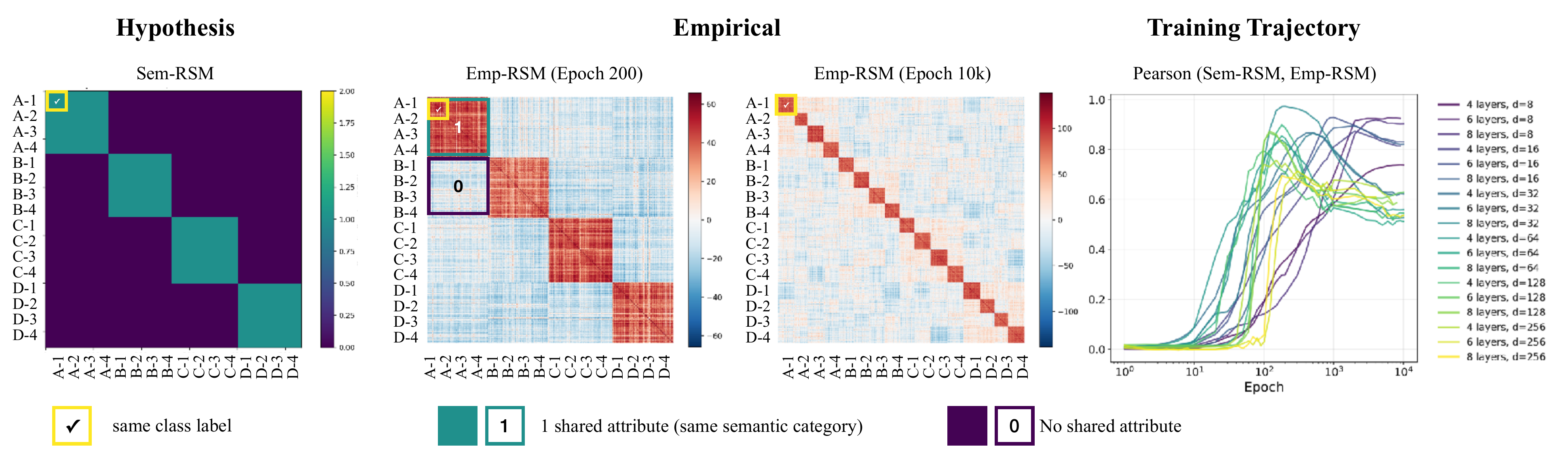}
    }
    \vspace{-0.15in}
\caption{
\textbf{Transient semantic geometry in the hierarchical grammar} (Experiment~3). Sample-level $n\times n$ RSMs over the full training set, ordered by output label ($n=2000$, $K=16$).
\textbf{Left:} Sem-RSM, where similarity is $1$ if two contexts share the same semantic category and $0$ otherwise.
\textbf{Middle:} empirical RSMs (Emp-RSM) at epochs 200 and 10k. Early in training, Emp-RSM exhibits clear category-level block structure matching Sem-RSM, indicating that contexts licensed by the same c-commanding semantic category are grouped together. By the end of training, this coarse semantic structure weakens and finer label-specific diagonal structure becomes more prominent.
\textbf{Right:} Pearson correlation between Emp-RSM and Sem-RSM over training, across model depths and widths. All models enter a strong semantic regime early, but late-stage behavior differs: narrow models ($d=8,16$) retain higher semantic alignment, whereas wider models peak earlier and decline more.
}
    \label{fig:language3_result}
    \vspace{-0.2in}
\end{figure}

Second, we provide mathematical justification for this optimization behavior. 
We prove two results (details in \Cref{sec:theorems}). \Cref{thm:transient_gram}: Assuming balanced output logits at initialization and initial features that are weakly $\eps$-correlated with the classifiers, the change in Emp-RSM after one step of gradient descent is biased towards the Sem-RSM provided $m>1$ contexts per label class and $\eps\ll 1/\sqrt{K}$. Second, \Cref{thm:gram_convergence}: assuming that gradient descent converges to a global minimizer, this minimizer possesses the ETF-RSM structure. Importantly, the spherical normalization is essential to guarantee this transition to ETF towards convergence; without it (in which case, it becomes vanilla BoW) the parameters converge to a semantic geometry. Thus, our spherical-BoW model can be interpreted as a middle-ground between vanilla BoW favoring only input-semantic geometry and an Unconstrained Feature Model (see \Cref{app:ufm}) favoring only output-label geometry.



\vspace{-0.1in}
\section{Conclusion, discussion, and limitations}
\vspace{-0.1in}

Our aim with this work is to start building an empirical and theoretical bridge between the theory around neural collapse in the classification regime, and the empirical realities of human language and language model training. 
In service of this, we have created three synthetic toy scenarios, which get progressively more language-like, and examine them in the representational collapse framework. We find that, even in cases where Simplex ETF-style representational collapse is theoretically predicted \textit{and} empirically arrived at in the terminal phase of training, \textbf{semantically-informed geometric organization is a transient phase in learning}. We show this to be the case for a toy category-matching classification, for a linear template language with Zipfian category semantics, and for a hierarchical grammar with category semantics. Our results show that the NC theory and realities of language models can in fact be empirically reconciled, showing how semantic geometry can arise early on in training even if models are large and not necessarily bottle-necked ($d\geq K$) and, thus, eventually collapse to output-label-driven NC geometries.

\textbf{How do our results transfer to the one-epoch fresh data regime?}\; Our experiments (as well as the NC literature more broadly) are based on scenarios where the model sees multiple epochs, which is not how LM training typically is. The use of only fresh data is likely to change both the theoretical and empirical realities of the setup.
In future work, we will aim to explore the questions of a) how do our results translate to cases of only fresh data (where we would need much larger vocabularies and model capacities)? and b) when does an undertrained, or transient state appear in a fresh data regime?

\textbf{Autoregressive training includes contexts of various lengths and uniqueness probabilities}\; In our original framing, we have pitched two settings against each other: the possibility that language model training data contains soft labels,
and the possibility where every context is unique (which we argue is overwhelmingly likely in long-context language modeling). However, this is not strictly true: in autoregressive language modelling 
the first few tokens (eg, the trigram $\rightarrow$ 4th word prediction) are very likely to actually be a soft-label scenario. The question remains: \textit{for a context of thousands of tokens, how important are the first $\sim$10 tokens, which might be in a soft-label regime}? 
Are these short contexts perhaps bootstrapping semantic representation, or is their effect essentially noise given their tiny fraction of the overall predicted tokens?
To understand this, we propose that future empirical work carefully control and examine the effects of the (sparse) soft-label aspect of LM training.


\textbf{Exploring more complex, realistic synthetic languages}\; In our languages, the latent structuring principle is a category-restriction semantic pattern. This way, we can confidently state a ground-truth Semantic RSM: one based on category overlap. However, the symbolic structuring features of language are much more complex, interfacing category restriction with many other syntactic and semantic latent structures. As we discuss in \Cref{sec:lang3}, integrating careful linguistic analysis with complex multifaceted synthetic languages will expand our results to better understand both the learning theory behind language models, as well as how different conceptions of linguistic structure interface with learnability. 


\textbf{Links to generalization} In image classification, the emergence of ETF geometry has been empirically linked to good generalization performance \citep{papyan2020prevalence}. \citet{wu2024linguistic} take early steps toward extending these links to the language setting, but do not account for the semantic structure of the input data, which as we show shapes representation geometry particularly in the early stages of training. Future work should investigate how the transient semantic geometry correlates with (or even facilitates) generalization.


{
\section*{Acknowledgments}
This work was partially funded by the NSERC Discovery Grant No. 2021-03677, the Alliance Grant ALLRP 581098-22, and a gift from Google. YZ was also supported by the UBC 4YF Doctoral Fellowship. The authors also acknowledge the use of the Sockeye cluster by UBC Advanced Research Computing.
}

\bibliographystyle{./colm2026_conference}
\bibliography{refs, ./new_refs, ./one-hot_refs}

\appendix

\section{Additional experimental settings and results}

\subsection{Experimental details}
\label{app:experimental-details}

\paragraph{Experiment 1.}
We sweep depth $L \in \{1,2,4,8\}$ and embedding dimension
$d \in \{4,6,8,16,32,64\}$.
All models use 8 attention heads.
Training is performed with full-batch Adam
($\mathrm{lr}=5\times10^{-4}$) for $10^4$ epochs.

\paragraph{Experiment 2.}
We sweep depth $L \in \{2, 4, 6\}$ and embedding dimension
$d \in \{32, 64, 128, 256\}$.
All models use 8 attention heads.
Training is performed with Adam, batch size $B=512$, learning rate $10^{-4}$, for $10^4$ epochs.

\paragraph{Experiment 3.}
We sweep depth $L \in \{4, 6, 8\}$ and embedding dimension
$d \in \{8, 16, 32, 64, 128, 256\}$.
All models use 8 attention heads.
Training is performed with Adam, batch size $B=512$, learning rate $5\times10^{-4}$, for $10^4$ epochs.

\subsection{Neural Collapse diagnostics}
\label{app:nc_diagnostics}

The RSA analyses compare the sample-level empirical RSM to the Sem-RSM and ETF/SELI-RSM. Since the ETF/SELI-RSM correlation is a similarity-structure diagnostic, we also report two standard Neural Collapse diagnostics, NC1 and NC2, which more directly measure within-class collapse and class-mean geometry \citep{papyan2020prevalence}.

Let $H = [h_{1,1}, \dots, h_{K,n_K}] \in \mathbb{R}^{d \times \sum n_k}$ be the feature matrix (last-layer embeddings) of examples belonging to $K$ classes, such that there is equal number of examples $n_1=\ldots=n_K$ per class. Let $W = [w_1, \dots, w_K]^\top \in \mathbb{R}^{K \times d}$ be the classifier weight matrix. Let $\mu_k = \frac{1}{n_k} \sum_{i=1}^{n_k} h_{k,i}$ be the $k$-th class mean and $\mu_G = \frac{1}{\sum n_k} \sum_{k,i} h_{k,i}$ be the global mean. 

$\mathcal{NC}1$ (Variability Collapse): The within-class covariance $\Sigma_W$ vanishes. $\text{Metric}_{\mathcal{NC}1} = \Sigma_W  \to 0$
where $\Sigma_W = \frac{1}{\sum n_k} \sum_{k=1}^K \sum_{i=1}^{n_k} (h_{k,i} - \mu_k)(h_{k,i} - \mu_k)^\top$.

$\mathcal{NC}2$ (Convergence to ETF): The centered class means $\tilde{H} = [\mu_1 - \mu_G, \dots, \mu_K - \mu_G]$ converge to a Simplex Equiangular Tight Frame (ETF).
\begin{small}
$$\text{Metric}_{\mathcal{NC}2} = \left\| \frac{\tilde{H}^\top \tilde{H}}{\|\tilde{H}^\top \tilde{H}\|_F} - \frac{1}{\sqrt{K-1}} \left( I_K - \frac{1}{K} {1}_K {1}_K^\top \right) \right\|_F \to 0$$
\end{small}
NC2 predicts that the representations of examples belonging to different classes are equally (in fact, maximally) separated. Here and throughout, we assume $d\geq K$. This assumption guarantees the existence of an ETF and it is required for NC. 

\begin{figure}[h]
    \centering
    \includegraphics[width=0.7\linewidth]{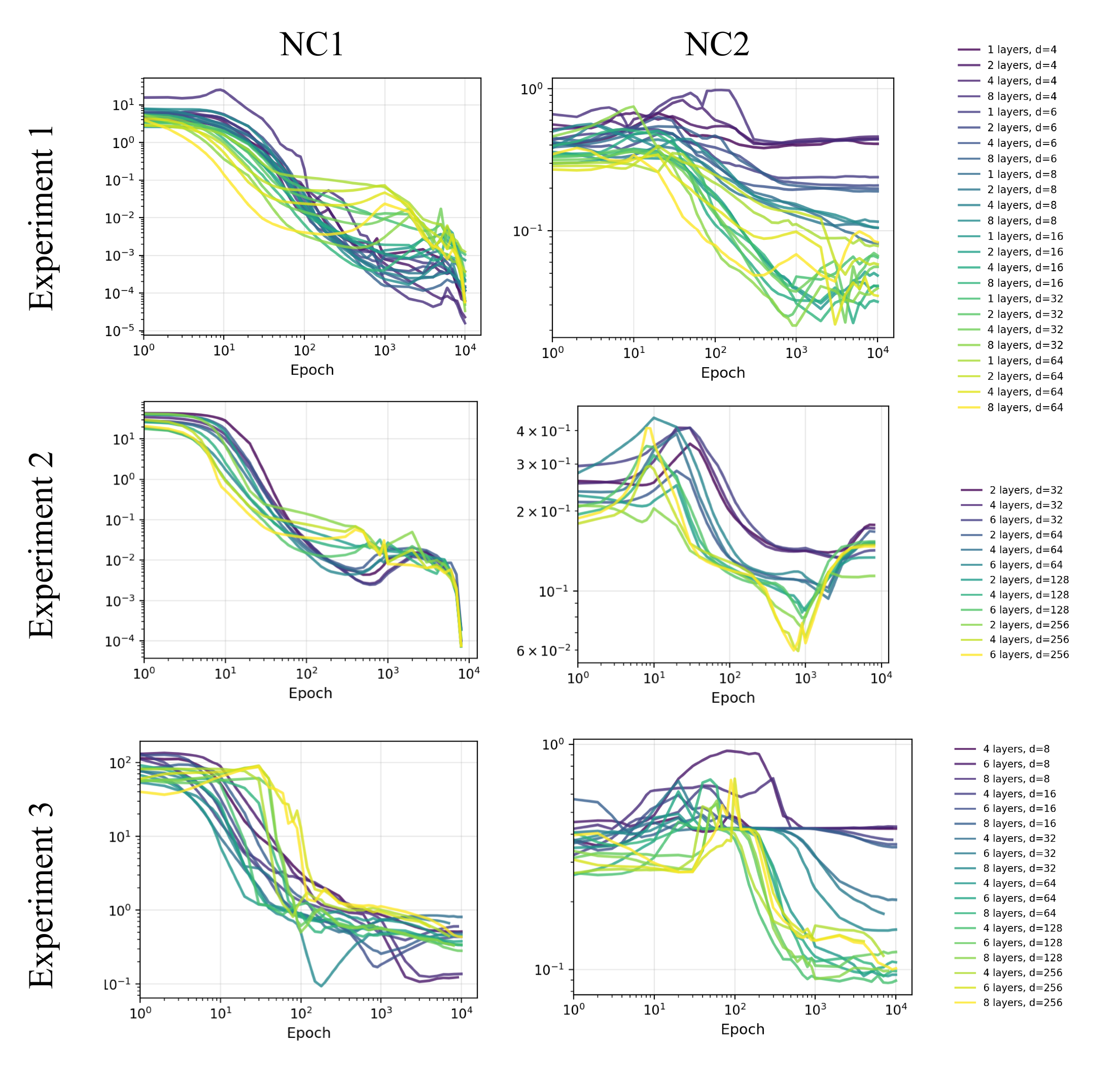}
    \vspace{-2.0em}
    \caption{
    Neural Collapse diagnostics across three experiments. NC1 measures within-class variability relative to class-mean variability, and NC2 measures the deviation of the centered class-mean Gram matrix from the corresponding ETF/SELI target geometry. Across experiments, NC1 decreases sharply during training, while NC2 generally decreases more strongly for larger models, indicating that label-driven structure becomes stronger later in training.
    }
    \label{fig:nc_diagnostics}
\end{figure}

Figure~\ref{fig:nc_diagnostics} shows that the NC diagnostics are broadly consistent with the RSA results. NC1 decreases over training, indicating increasing within-class collapse. NC2 also tends to decrease, especially in larger models, indicating movement toward the label-driven class-mean geometry. Experiment~2 is the least direct comparison because the label distribution is imbalanced; there we compare against the SELI target rather than the balanced ETF target.

\subsection{Numerical characterization of semantic structure}
\label{app:numerical}


{Motivated by the standard NC1 metric in \Cref{app:nc_diagnostics}, we also consider a direct numerical measure of semantic clustering. Instead of grouping examples by output label, this metric groups examples by latent semantic attribute, and asks whether embeddings that share the same attribute form tight clusters.}

\textbf{Semantic Convergence (\(\mathcal{NC}1_{\mathrm{sem}}\)).}
We compute the standard \(\mathcal{NC}1\) metric with one modification: instead of grouping samples by their label class \(y_k\), we group them by their semantic class (attribute) \(a \in \mathcal{A}_F\). Let \(\mu^{(a)}\) denote the centroid of the semantic class \(\mathcal{S}_a\), and let \(\mu_G^{(F)}\) denote the global mean over all embeddings associated with factor \(F\). We define the within-class and between-class scatter matrices as
\[
\Sigma_W(F)
=
\frac{1}{N}
\sum_{a\in\mathcal{A}_F}
\sum_{h\in\mathcal{S}_a}
(h-\mu^{(a)})(h-\mu^{(a)})^\top,
\]
\[
\Sigma_B(F)
=
\frac{1}{|\mathcal{A}_F|}
\sum_{a\in\mathcal{A}_F}
(\mu^{(a)}-\mu_G^{(F)})(\mu^{(a)}-\mu_G^{(F)})^\top.
\]

We quantify the collapse by the trace ratio $\mathcal{NC}1_{sem}(F) = \text{Tr}(\Sigma_{W}) / \text{Tr}(\Sigma_{B})$. 
A lower score indicates that sample embeddings with shared attributes have clustered tightly around a common centroid.

\textbf{The semantic alignment ratio ($\mathcal{NC}1_\text{sem}$).} We normalize our observation against the expected terminal value of the ETF baseline:\begin{equation}\mathcal{R}_\text{sem}(F) = \frac{\mathcal{NC}1_\text{sem}(F)}{\mathcal{NC}1_{\text{ETF}}(F)}\end{equation}
This ratio serves as our primary order parameter. $\mathcal{R}_\text{sem} < 1.0$ identifies the Semantic Regime, where geometry is more structured than the UFM baseline. $\mathcal{R}_\text{sem} \approx 1.0$ indicates the Collapsed Regime, where geometry is dictated by label class. We also validate these regimes qualitatively by inspecting the Gram matrix $G = H^\top H$ with rows/columns sorted by attribute.

We present the corresponding results in \cref{fig:appendix-numerical}. They show the same qualitative pattern as our RSA analyses: semantic structure peaks early in training, while the label-driven collapsed geometry strengthens later. Its limitation is that it becomes confounded once the embedding dimension is constrained. Small \(d\) naturally compresses all embeddings into a low-dimensional subspace, which can reduce within-cluster scatter even when the geometry is not specifically organized by the semantic attributes of interest. In other words, \(\mathcal{NC}1_{\mathrm{sem}}\) measures whether semantic groups are tight, but not whether this tightness reflects the correct semantic structure rather than generic low-dimensional compression. This makes the metric less reliable for comparing models across different embedding dimensions. For this reason, we do not use it as our main measure, and instead rely on Hypothesis RSA, which better distinguishes true semantic organization from trivial low-rank compression.

\begin{figure}
    \centering
    \includegraphics[width=\linewidth]{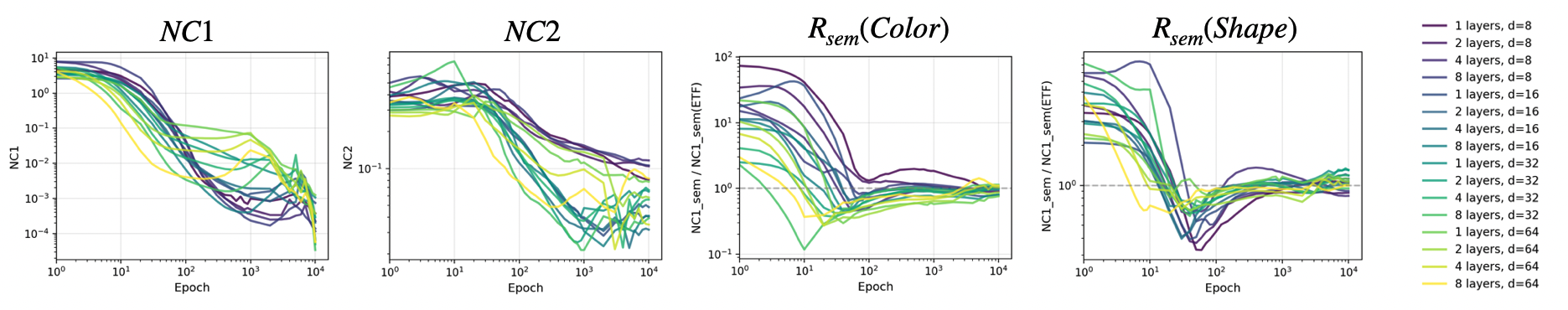}
    \caption{Recreation of \cref{fig:main_results}, right panel, using our numerical measures of neural collapse and semantic representation. We see very similar learning dynamics to the patterns we report in the main text, with ETF-style collapse increasing monotonically, while semantic representation peaks early in training.}
    \label{fig:appendix-numerical}
\end{figure}

\subsection{Different-label semantic structure}
\label{app:cross_label_diagnostics}

One possible concern with the Sem-RSM correlation is that it can partly benefit from same-label structure. In our synthetic languages, two examples with the same output label also share all latent factors, so they are maximally similar under the Sem-RSM. Therefore, a representation that mainly clusters examples by exact output label can still have nonzero correlation with the semantic hypothesis. To check that our early semantic signal is not only coming from this effect, we measure semantic structure using only pairs of examples with different output labels.

Let $G_{\mathrm{emp}}^{(t)}$ be the empirical RSM at checkpoint $t$, $G_{\mathrm{sem}}$ be the semantic RSM, and $y_i$ be the output label of example $i$. We define the set of different-label pairs as
\[
\mathcal{D} = \{(i,j): i<j, \; y_i \neq y_j\}.
\]
We also split this set into pairs that share at least one latent factor and pairs that share none:
\[
\mathcal{D}_{+} = \{(i,j)\in\mathcal{D}: G_{\mathrm{sem},ij}>0\},
\qquad
\mathcal{D}_{0} = \{(i,j)\in\mathcal{D}: G_{\mathrm{sem},ij}=0\}.
\]
First, we compute a different-label semantic gap:
\[
\Delta_{\mathrm{diff}}(t) = \frac{1}{|\mathcal{D}_{+}|} \sum_{(i,j)\in\mathcal{D}_{+}} G_{\mathrm{emp},ij}^{(t)} - \frac{1}{|\mathcal{D}_{0}|} \sum_{(i,j)\in\mathcal{D}_{0}} G_{\mathrm{emp},ij}^{(t)} .
\]
This asks whether two examples with different labels are still closer when they share some latent semantic factor than when they share none.

Second, we compute semantic RSA after removing same-label pairs. Here,
\(\operatorname{corr}(\cdot,\cdot)\) denotes Pearson correlation, matching the RSA metric
used in the main text, and the computation is restricted to the different-label pair set \(D\):
\[
\rho_{\mathrm{diff}}(t)
=
\operatorname{corr}\left(
\left\{G^{(t)}_{\mathrm{emp},ij}\right\}_{(i,j)\in D},
\left\{G_{\mathrm{sem},ij}\right\}_{(i,j)\in D}
\right).
\]
This is the same semantic-RSM comparison as in the main text, but restricted to different-label pairs.

\begin{figure}[h]
    \centering
    \includegraphics[width=0.7\linewidth]{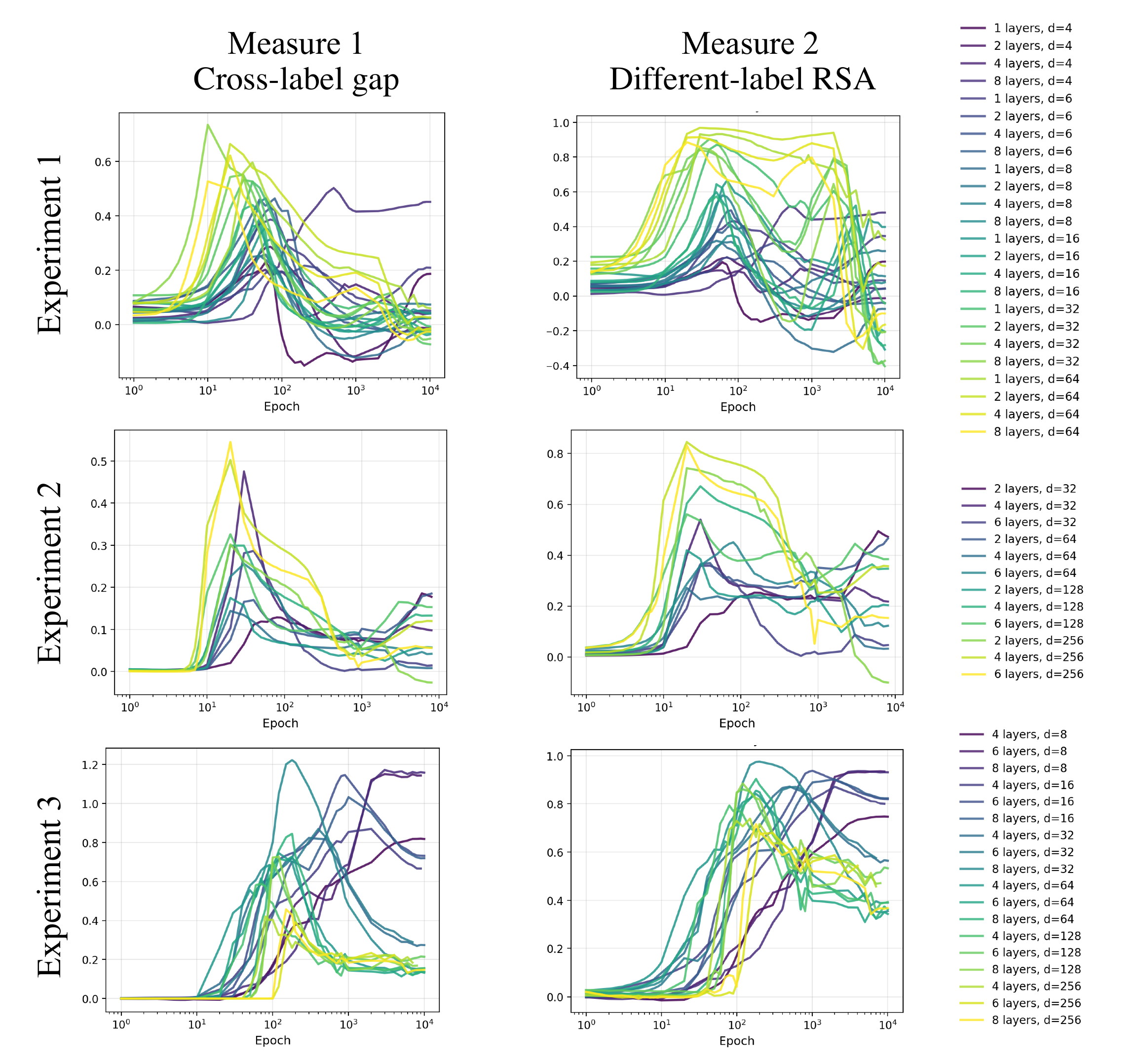}
    \vspace{-2.0em}
    \caption{
    Cross-label semantic diagnostics for all three experiments.
    The left column (measure 1) shows the cross-label semantic gap between different-label pairs that share at least one latent factor and those that share none.
    The right column (measure 2) shows semantic RSA computed only over different-label pairs. Lighter colors correspond to larger models. Across experiments, both measures rise early in training and weaken later, especially in larger models, showing that the early semantic alignment is not explained by same-label collapse alone.
    }
    \label{fig:rebuttal_cross_label_diagnostics}
\end{figure}

Figure~\ref{fig:rebuttal_cross_label_diagnostics} shows that both different-label diagnostics follow the same qualitative pattern as the main semantic RSA curves. They rise early in training, and then decrease later as the representations become more organized by exact output label. This supports the interpretation that the early semantic regime reflects graded structure across labels, not only collapse within each label.

\subsection{SELI geometry for imbalanced label classes}
\label{}
As mentioned in the main text, in the imbalanced setting, the balanced ETF geometry is no longer the ETF geometry \cite{fang2021exploring}. Instead, we use its imbalanced generalization, the SELI geometry \citep{thrampoulidis2022imbalance}. 

For ease of reference we briefly explain here how to obtain this geometry: Let $Y_c=(I_K-\frac{1}{K}1_K1_K^\top)Y$ be the centered one-hot encoding label matrix (a simplex-encoded label, or SEL, matrix), and let $Y_c=U\Sigma V^\top$ be its SVD.
The SELI-RSM reference geometry is then described by the Gram matrix  $G_\text{SELI}=V\Sigma V^\top$, arising from the logits interpolating the SEL matrix.

This geometry has been empirically verified and theoretically justified in STEP-imbalanced image classification \cite{thrampoulidis2022imbalance, behnia2023implicit}; here we adopt it as the best-known proxy in the literature for more general imbalanced settings. For consistency, and to emphasize the label-driven nature of NC geometries, we retain the name ETF-RSM throughout the main body, but in Experiment 2 this reference should be understood as the SELI generalization described above.




\subsection{Additional results for Experiment 2}
\label{app:additional-exp2}

\begin{figure}
    \centering
    \makebox[\linewidth][c]{
    \includegraphics[width=1.2\linewidth]{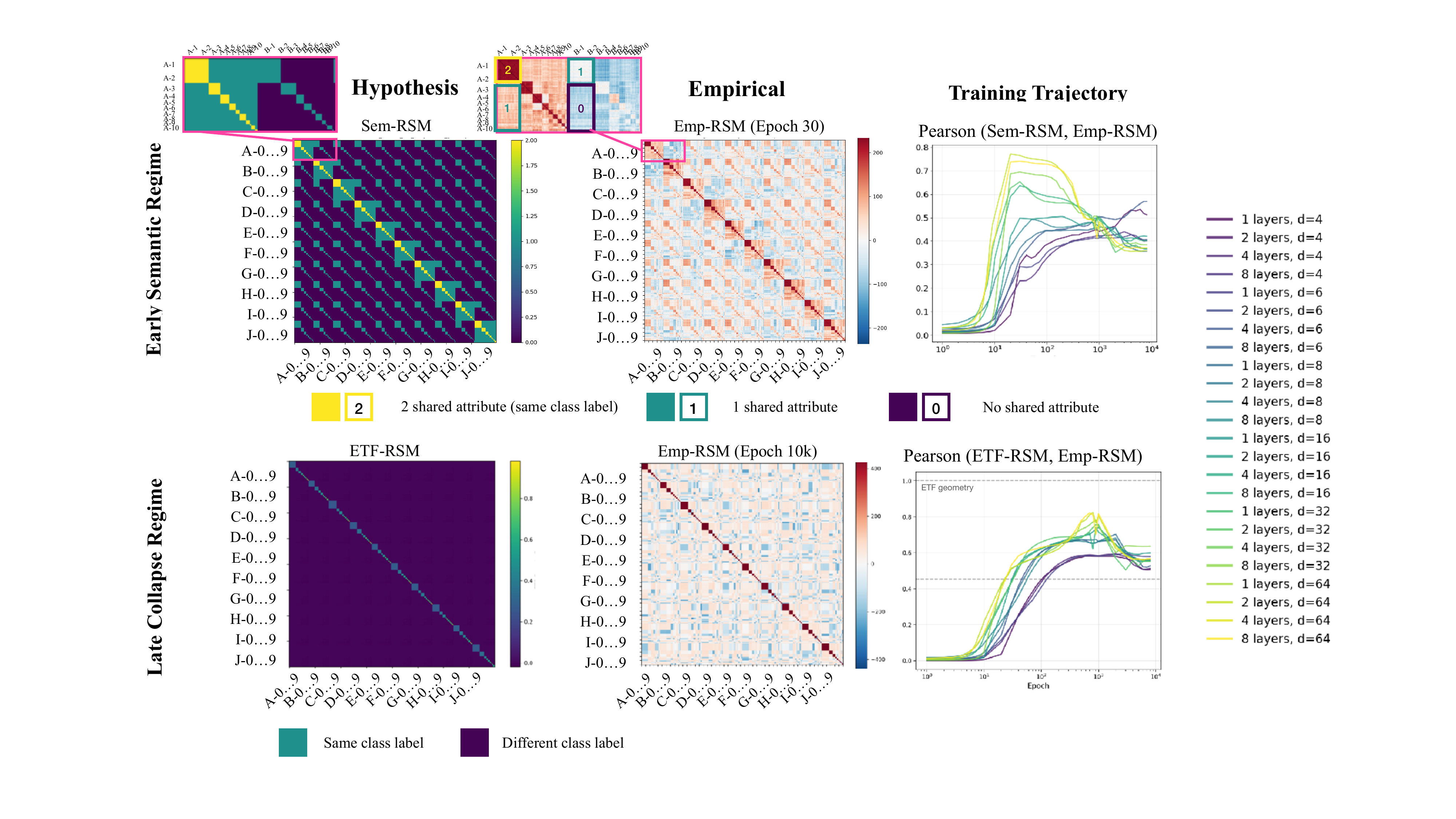}
    }
    \caption{\textbf{Full representational similarity results for Experiment 2.}
This figure is the Experiment~2 analogue of \cref{fig:main_results}. \textbf{Left column:} hypothesis RSMs. Sem-RSM encodes the semantic hypothesis, where two contexts have similarity 2 if they share both latent attributes, 1 if they share one, and 0 otherwise. ETF-RSM encodes the label-driven reference geometry. In this imbalanced setting, this reference should be understood as the SELI generalization of ETF described in the main text. \textbf{Middle column:} empirical RSMs at an early epoch 30 and a late epoch 10k. Early in training, the empirical RSM closely matches the graded block structure of Sem-RSM. Later in training, the matrix becomes more label-driven, but some off-diagonal structure remains. \textbf{Right column:} Pearson correlation between Emp-RSM and each hypothesis RSM over training, across model depths and widths. As in Experiment~1, alignment with Sem-RSM rises early and then declines, while alignment with ETF-RSM increases later. Wider models show both stronger early semantic alignment and a stronger later shift toward the label-driven geometry.}
    \label{fig:language2_result_full}
\end{figure}

\Cref{fig:language2_result_full} shows the full Experiment~2 trajectories. The clearest comparison is across model size: wider models reach stronger early alignment with Sem-RSM and also move further toward the label-driven geometry later in training, while depth has a smaller effect.

The full figure also clarifies the late-stage geometry. The empirical RSM becomes more label-driven over training, but it does not become diagonal-only. Instead, some off-diagonal structure remains, which is consistent with the Zipfian label imbalance in this experiment. Since the terminal label-driven geometry is the imbalanced SELI geometry rather than the balanced ETF pattern, this residual structure should be understood as part of the late imbalance-driven regime.
\subsection{Additional results for Experiment 3}

\cref{fig:language3_result_full} shows the full Experiment~3 results in the same format as \cref{fig:main_results}. All models first enter the semantic regime: early in training, the empirical RSM matches Sem-RSM much more closely than ETF-RSM. Later, the empirical RSM becomes more diagonal, Sem-RSM correlation drops, and ETF-RSM correlation increases. But compared with Experiment~1, this later shift is weaker, and semantic blocks remain visible even at epoch~10k.

The clearest comparison in this figure is across widths. Smaller-width models keep higher Sem-RSM correlation and lower ETF-RSM correlation through late training. Models with widths $d=64$ and $d=128$ show a more clearly transient semantic phase: they peak higher on Sem-RSM early, then decline more, while also moving further toward ETF-RSM. Depth has a smaller effect than width. Overall, in Experiment~3, width mainly controls how far training moves away from the early semantic geometry.
\label{app:additional-exp3}

\begin{figure}
    \centering
    \makebox[\linewidth][c]{
    \includegraphics[width=1.2\linewidth]{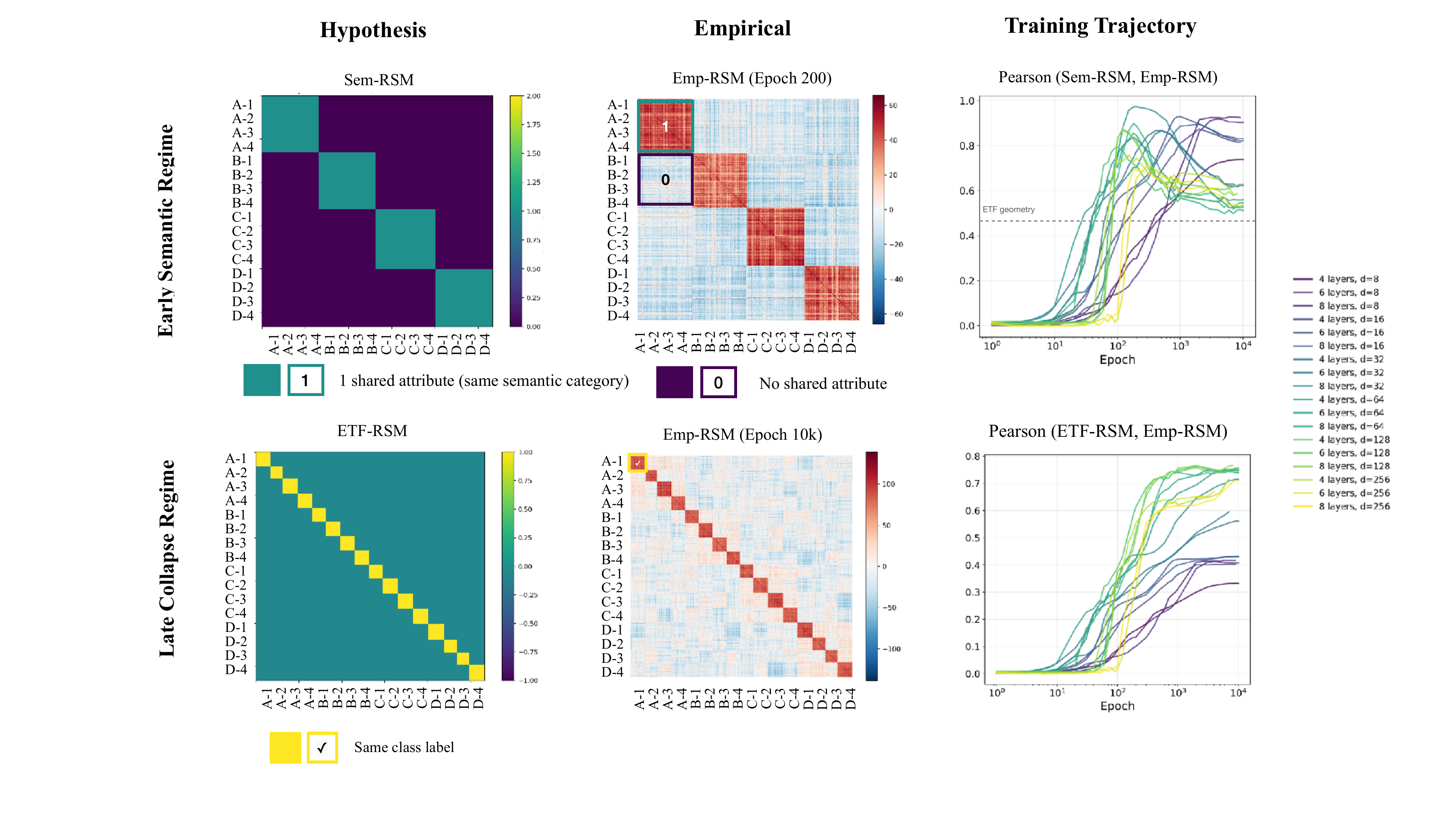}
    }
\caption{\textbf{Full representational similarity results for Experiment 3.}
This figure is the Experiment~3 analogue of \cref{fig:main_results}. \textbf{Left column:} hypothesis RSMs. Sem-RSM encodes the semantic hypothesis, where two contexts have similarity 1 if they are associated with the same semantic category and 0 otherwise. ETF-RSM encodes the label-based collapsed geometry, with $\checkmark$ marking same-label pairs. \textbf{Middle column:} empirical representational similarity matrices (Emp-RSM) at an early epoch 200 and a late epoch 10k. At epoch 200, the empirical RSM shows clear semantic block structure that closely matches Sem-RSM. At epoch 10k, this block structure is weaker and the matrix is more diagonal, although the semantic blocks are still faintly visible. \textbf{Right column:} Pearson correlation between Emp-RSM and each hypothesis RSM over training, across model depths and widths. As in the earlier experiments, correlation with Sem-RSM rises early and correlation with ETF-RSM increases later in training. But the dynamics differ from Experiment~1. Models with intermediate and larger widths, especially $d=64$ and $d=128$, reach stronger early alignment with the semantic hypothesis and also align more with ETF-RSM later in training. Smaller-width models retain relatively high semantic alignment until the end of training, indicating a less complete shift away from the semantic geometry.}
    \label{fig:language3_result_full}
\end{figure}

\subsection{Layerwise geometry}
\label{app:layerwise_geometry}

The main experiments measure representation geometry at the final layer, since this is the representation directly used by the readout. We additionally ask whether the same semantic-to-label transition appears uniformly across layers, or whether different layers preserve different parts of the geometry.

{For each checkpoint $t$ and layer $\ell$, we construct an empirical RSM from the layer-$\ell$ representations, denoted $G_{\mathrm{emp},\ell}^{(t)}$. We then repeat the two diagnostics defined in \cref{app:cross_label_diagnostics} for each layer: the different-label semantic gap $\Delta_{\mathrm{diff},\ell}(t)$ and the correlation with the label-driven RSM, $\rho_{\mathrm{diff},\ell}(t)$.}

\begin{figure}[h]
    \centering
    \includegraphics[width=1\linewidth]{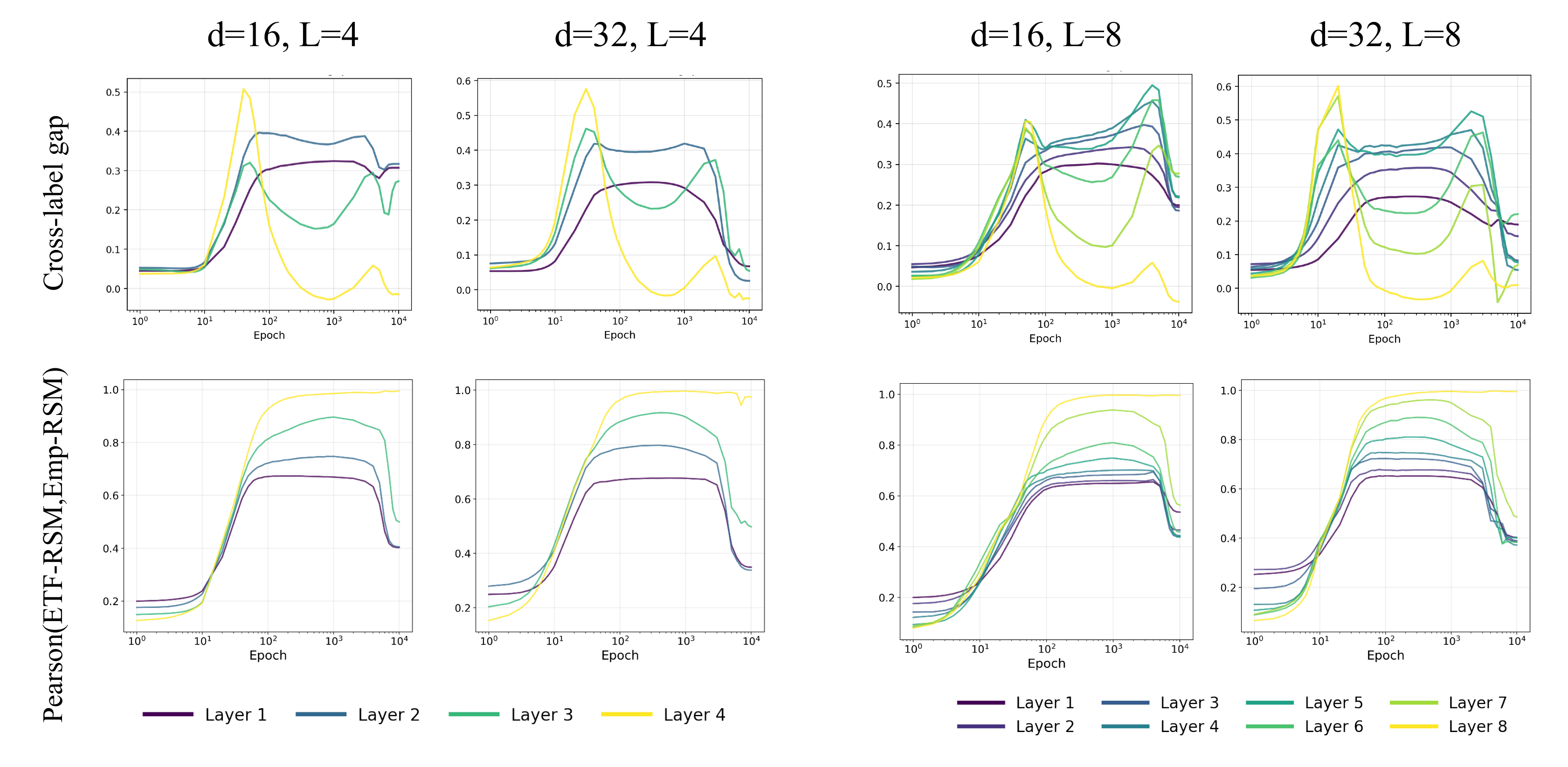}
    \vspace{-1.2em}
\caption{Layerwise RSA dynamics for Experiment 1. For four representative models, we track the cross-label semantic gap (top row), which compares the similarity of different-label pairs that share at least one latent factor to those that share none, and the Pearson correlation between each layer's empirical RSM and the ETF-RSM (bottom row) over training. Colors indicate layer depth, with darker curves corresponding to earlier layers and lighter curves to later layers. ETF-like output-label geometry is concentrated near the readout, while semantic structure can persist in earlier and intermediate layers. } \label{fig:layerwise_exp1}
\end{figure}

\begin{figure}[h]
\centering
\makebox[\linewidth][c]{%
\includegraphics[width=1.15\linewidth]{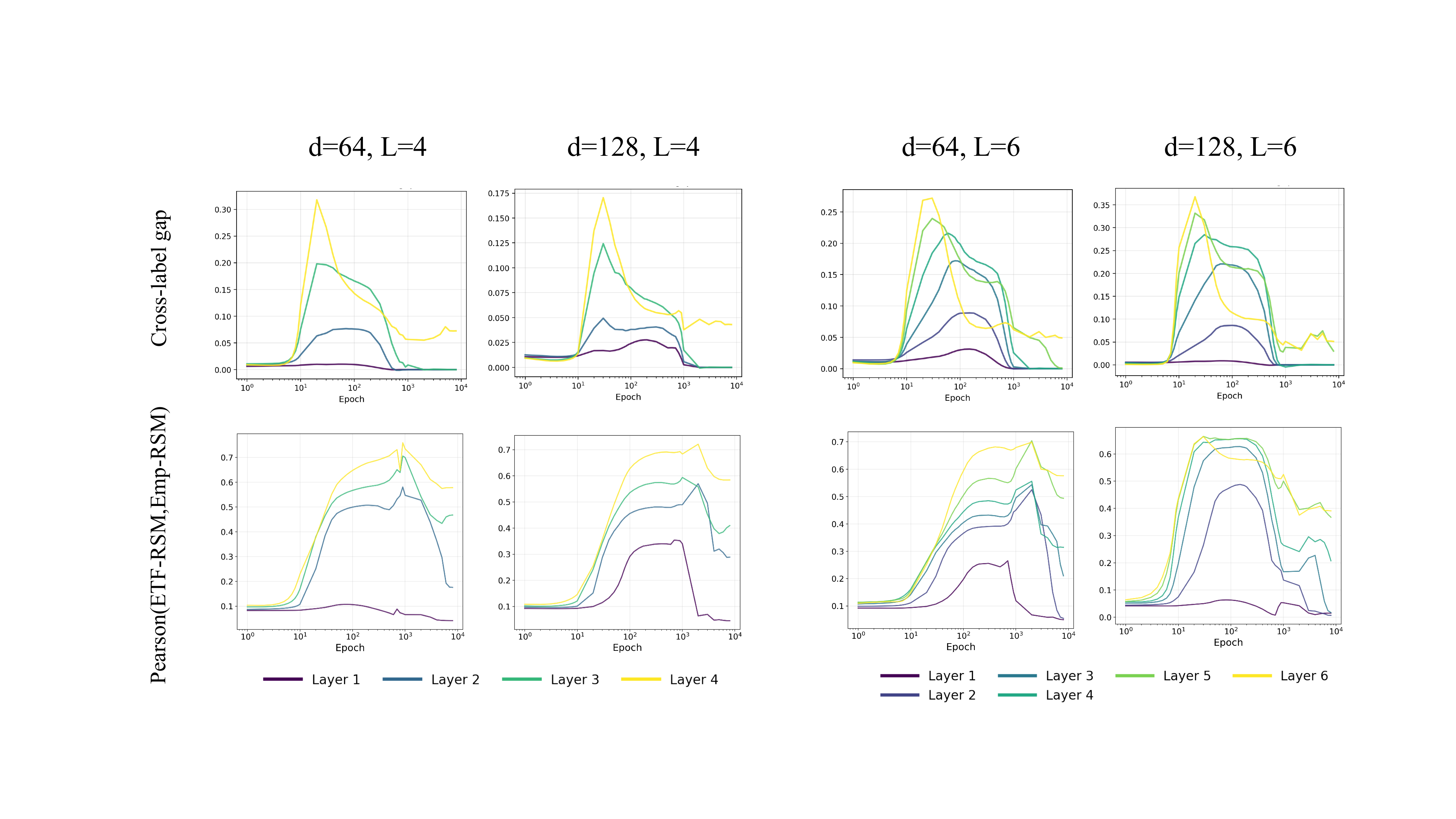}%
}
\vspace{-1.3cm}
\caption{Layerwise RSA dynamics for Experiment 2.}
\label{fig:layerwise_exp2}
\end{figure}

\begin{figure}[h]
    \centering
    \makebox[\linewidth][c]{%
    \includegraphics[width=1.15\linewidth]{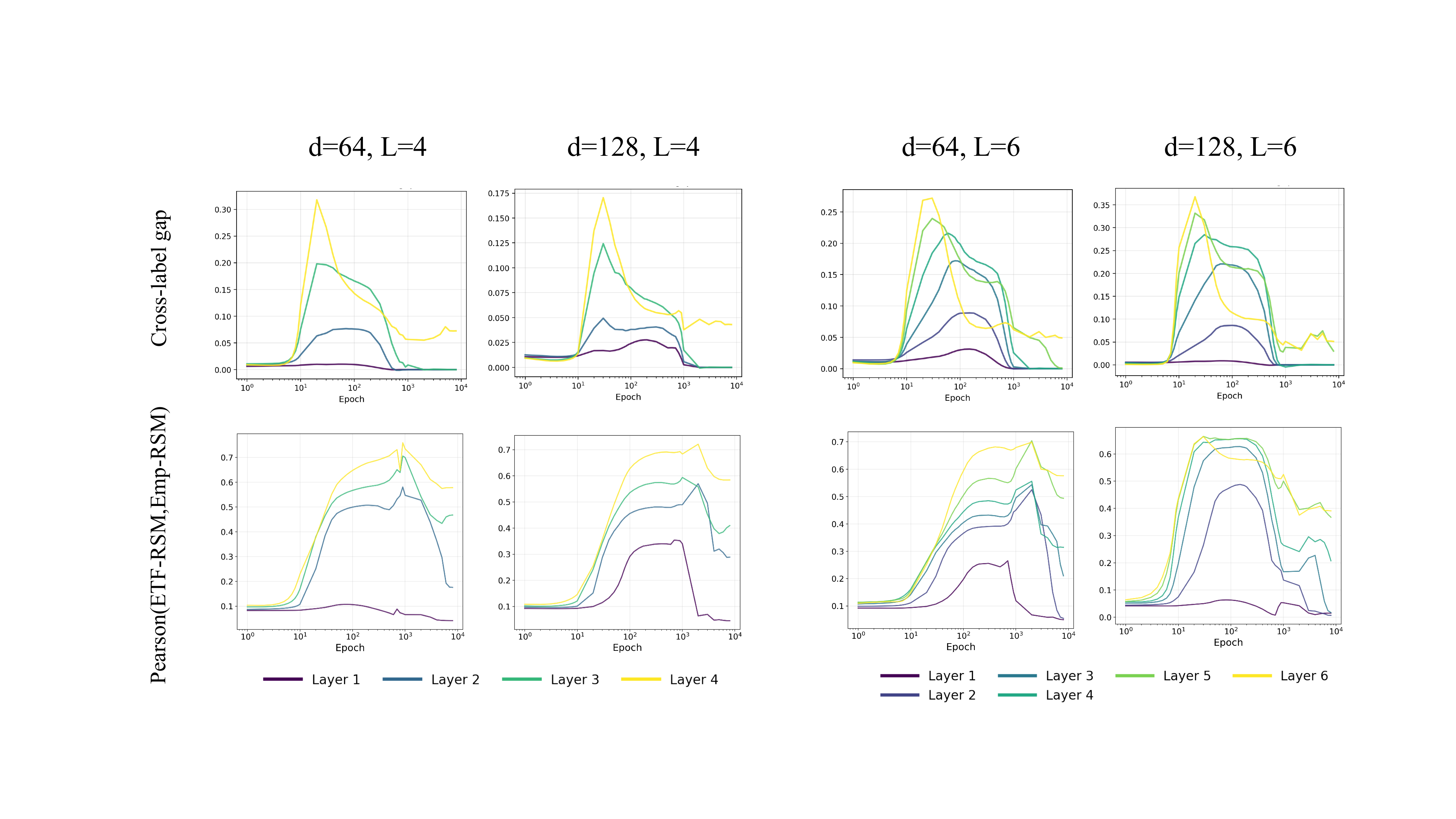}
    }
    \vspace{-1.2cm}
\caption{Layerwise RSA dynamics for Experiment 3.} \label{fig:layerwise_exp3}
\end{figure}

\Cref{fig:layerwise_exp1}--\cref{fig:layerwise_exp3} show that the transition from semantic structure to label-driven structure is not identical at every layer. Across experiments, the ETF-RSM correlation is generally strongest in the later layers, which is expected because these layers are closest to the output prediction. The semantic gap is more mixed. In Experiment~1, it can remain visible in earlier and intermediate layers even as the final layer becomes more label-driven. In Experiments~2 and~3, upper layers often carry both a strong semantic gap and a strong label-driven signal. Thus, the layerwise picture is not simply that early layers are semantic and late layers are label-driven. Rather, the upper layers appear to first organize examples by task-relevant latent structure, and then become increasingly aligned with exact output-label geometry later in training.

This layerwise analysis does not change the main conclusion, but it makes the transition more precise: the final representation is where label-driven collapse is most visible, while intermediate representations can retain more mixed structure. Future work could study this structure in greater detail.

\section{Theoretical Analysis and Proofs} 
\label{app:theory}

\paragraph{Notation.} We introduce necessary notation that we use throughout this section. For positive integer $N$, we write $[N]:=\{1,2,\ldots,N\}$. $I_N$ denotes the identity matrix of size $N$. $1_N$ is the vector of all-ones of size $N$. $\otimes$ denotes Kronecker product. $\sft{\cdot}:\R^N\rightarrow\Delta^{N-1}$ denotes the softmax map and $\sfti{c}{v}=\exp(v_c)/{\sum_{c'\in[n]}\exp(v_{c'})}$ is its $c$-th entry. When $\sft{\cdot}$ is applied on a matrix, it acts columnwise. For vector $v$, we denote its Euclidean (L2) norm as $\|v\|_2$. We let $\lnorm{v}=v/\|v\|_2$ denote spherical normalization. When $\lnorm{\cdot}$ is applied to a matrix, it acts columnwise. For matrix $V$, we denote its operator norm as $\|V\|_\text{op}$ and its Euclidean/Frobenius (L2) norm as $\|V\|_F$. We let $e_i$ denote the standard $i$-th basis vector.

\subsection{Details on our Spherical BoW}

Let $K$ be the number of classes and $m$ be the number of samples per class. The total number of samples is $n = K \cdot m$. Let $s$ be the number of semantic features. Denote $V = n + s$ the total number of features. Let $B \in \mathbb{R}^{s \times K}$ be the base matrix of shared semantic features. For example, for $s=6$ semantic features corresponding, in order, to red, blue, circle, square, triangle, star, and $K=8$ classes corresponding in order to red-circle, red-square, and so on, the semantic base matrix $B\in\R^{6\times 8}$ reads 
\begin{align}\label{eq:B matrix}
B = \begin{bmatrix}
    1 & 1 & 1 & 1 & 0 & 0 & 0 & 0 \\
    0 & 0 & 0 & 0 & 1 & 1 & 1 & 1 \\
    1 & 0 & 0 & 0 & 1 & 0 & 0 & 0 \\
    0 & 1 & 0 & 0 & 0 & 1 & 0 & 0 \\
    0 & 0 & 1 & 0 & 0 & 0 & 1 & 0 \\
    0 & 0 & 0 & 1 & 0 & 0 & 0 & 1
\end{bmatrix}\,.
\end{align}
With this, the data matrix $X \in \mathbb{R}^{V \times n}$ is constructed by combining unique one-hot features for every sample with $m$ repeated copies of the shared class features, i.e.,
\[
X = \begin{bmatrix}
    I_n \\
    B \otimes {1}_m^\top
\end{bmatrix}\in\R^{V\times n}\,,
\]
where $I_n$ is the $n \times n$ identity matrix and $1_m$ is the all-ones vector of length $m$. Let $x_i \in \mathbb{R}^V$ denote the $i$-th column of $X$. 

Let $y_i \in[K]$ be the true class of sample $i$. Specifically, let the one-hot encoding matrix $Y = I_K \otimes 1_m^\top$. We optimize a feature weight matrix $H \in \mathbb{R}^{d \times V}$ and let a classifier weight matrix $W \in \mathbb{R}^{K \times d}$. Unless otherwise stated, we set $d=K$.  

We consider the following training objective
\begin{align}\label{eq:nBOW}
\bar{\Lc}(H, W) = \mathrm{KL}\left(Y||\text{Softmax}({W\cdot\lnorm{HX}})\right)= - \sum_{i=1}^n\log \left( \sfti{y_i}{W \frac{H x_i}{\|H x_i\|_2}} \right)\,. 
\end{align}
This is essentially a continuous Bag-of-words (BoW) model with an additional projection $\lnorm{Hx_i}=Hx_i/\|Hx_i\|_2$ of the  learned features $Hx_i$  onto the unit hypersphere before classification. The normalization can be interpreted as  modeling 
 the layer-norm at the last layer of the architecture. 
Also, consistent with the CBoW setup \citep{mikolov2013efficient}, note that optimization over the parameters $W$ and $H$  imposes no further architectural constraints. We call the model in Eq. \eqref{eq:nBOW} the Spherical BoW model.

\begin{remark}[Relation to UFM]\label{app:ufm}
    Our model is inspired by, but differs  from the unconstrained features model (UFM) \citep{mixon2020neural,fang2021exploring} which is the go-to model in the NC literature. In the UFM, the penultimate layer features of a deep-net are treated as free optimization variables rather than the output of a specific architecture. In the language setting that we study here, this means that embedding of each context is a free optimization variable \citep{zhao2024implicit}. To explicitly relate it to Eq. \eqref{eq:nBOW}, the UFM corresponds to minimizing jointly over $W$ and $H$ the loss $\Lc(H,W)=\mathrm{KL}\left(Y||\sft{WH}\right)=-\sum_{i=1}^n\log(\sfti{y_i}{WHe_i})$.  Because the UFM treats embeddings as free optimization variables, we see that it  abstracts away the input data. This abstraction precludes any analysis of how specific structure of the input data influences the resulting representations. Instead, our BoW-style model explicitly accounts for such structures by replacing the basis input vector $e_i$ above with a one-hot encoding $x_i$ over \emph{multiple features} present in the $i$-th example. 
\end{remark}

\begin{remark}[Normalization] The spherical normalization is essential in Eq. \eqref{eq:nBOW} to guarantee that in the TPT the geometry converges to an ETF. In the absense of it, we can show that the geometry retains semantic structure even at convergence. Thus, the vanilla CBoW model does not exhibit the transient semantic geometry seen in our transformer experiments. 
\end{remark}

\subsection{Analysis}\label{sec:theorems}
To aid the analysis, we henceforth fix the classifier $W$ to an ETF, i.e. $W=I_K-\frac{1}{K}1_K1_K^\top$ and train $H$. We numerically verify in Fig. \ref{fig:cufm_language1_result} that fixing $W$ maintains the transient semantic-geometry behavior. The theorems below provide theoretical justification for this. Note that despite fixing the classifier, the training objective remains non-convex due to the spherical normalization, and so the gradient-descent analysis remains challenging. We therefore resort to certain simplifying conditions under which our formal results hold. Relaxing these conditions is left to future work; we note that gradient-descent dynamics and implicit bias in linear models are substantially less understood in the presence of normalization (e.g., \cite{hoffer2018norm,cao2023implicit,chou2024robust}) than without it \cite{soudry2018implicit}.

Let $Z = HX \in \mathbb{R}^{K \times N}$ be the unnormalized feature matrix, where $z_i \in \mathbb{R}^K$ is its $i$-th column. The normalized features are denoted as $u_i = \lnorm{z_i} = \frac{z_i}{\|z_i\|_2} \in \mathbb{S}^{K-1}$. Also, let $w_c, c\in[K]$ denote the $c$-th row of $W$.

\subsubsection{Early Phase}

\begin{theorem}[Semantic bias in the first gradient step]
\label{thm:transient_gram}
Consider one gradient descent step on the Spherical Bag-of-Words objective \eqref{eq:nBOW}
with fixed ETF classifier $W = I_K - \frac{1}{K}1_K1_K^\top$. Let $Z^{(1)}=H^{(1)}X\in\mathbb{R}^{K\times n}$ be the feature matrix after this step with learning rate $\eta$.
Assume at initialization there exists a small $\varepsilon > 0$ and a scale $\alpha > 0$ such that for all $i \in [n]$:

\begin{enumerate}[leftmargin=*]
    \item the initial logits are approximately balanced across samples, so that the corresponding softmax vectors satisfy:
$
    \Big\|\sft{Wu_i^{(0)}} - \frac{1}{K}1_K \Big\|_2
    \le \varepsilon$
    \item the initial feature norms are approximately constant:
    $
    \left| \frac{1}{\|z_i^{(0)}\|_2} - \alpha \right| \le \varepsilon ;
    $
    \item the initial normalized features are weakly aligned with their true classifiers:
    $
    \left| u_i^{(0)\top} w_{y_i} \right| \le \varepsilon \,.
    $
\end{enumerate}

Then, the change $Z^{(1)}-Z^{(0)}$ in the feature matrix after one gradient step is given by
\[
Z^{(1)} - Z^{(0)} = \alpha\eta W(I_K + m B^\top B) \otimes 1_m^\top + \eta R\,,
\]
where the residual is bounded by $\|R\|_{\mathrm{op}} \le \mathcal{O}(\varepsilon \sqrt{K}) \cdot \sqrt{m} ( 1 + m \|B\|_{\mathrm{op}}^2 )$. 
Thus, if $m>1$ and $\varepsilon \ll \frac{1}{\sqrt{K}}$, the relative error vanishes and the gradient update exhibits an explicit semantic bias through the term $B^\top B$.
\end{theorem}

\begin{proof}
Let $\ell_i(z_i) = -\log(\sfti{y_i}{W\lnorm{z_i}})= -\log(\sfti{y_i}{Wu_i})$ denote the contribution to the loss of the $i$-th example. For each  $i\in[n]$,
\[
\nabla_{z_i}\ell_i(z_i)  
=
\frac{1}{\|z_i^{(0)}\|_2}
\left(I_K-u_i^{(0)}u_i^{(0)\top}\right)
W (\sft{Wu_i^{(0)}}-e_{y_i}).
\]
By Assumption (1) and using $W 1_K = 0$, this becomes:
\[
\nabla_{z_i}\ell_i(z_i)
=
\frac{1}{\|z_i^{(0)}\|_2} \left(I_K-u_i^{(0)}u_i^{(0)\top}\right) (-w_{y_i} + W r_i)\,.
\]
By Assumption (2), $\alpha_i:= \frac{1}{\|z_i^{(0)}\|_2} = \alpha + \eps_{\alpha,i}$ with $|\eps_{\alpha,i}| \le \varepsilon$. Thus, 
\[
\nabla_{z_i}\ell_i(z_i) = -\alpha w_{y_i} + \xi_i\,,
\]
where 
\[
\xi_i = -\eps_{\alpha,i} w_{y_i} + \alpha_i (u_i^{(0)\top} w_{y_i}) u_i^{(0)} + \alpha_i (I_K - u_i^{(0)}u_i^{(0)\top}) W r_i\,.
\]
Noting that $\|w_{y_i}\|_2 < 1$, $\|u_i^{(0)}\|_2 = 1$, and both $W$ and the projection matrix have operator norms of $1$, applying the triangle inequality alongside Assumptions (1)--(3) yields
\[
\|\xi_i\|_2 \le \varepsilon + (\alpha + \varepsilon)\varepsilon + (\alpha + \varepsilon)\varepsilon = \varepsilon(1 + 2\alpha + 2\varepsilon) := \varepsilon'\,.
\]
Therefore, recalling $Y = I_K \otimes 1_m^\top$, the full gradient with respect to $Z$ is:
\[
\nabla_Z \bar{\mathcal L}(Z^{(0)})
=
-\alpha WY + \Xi,
\qquad
\|\Xi\|_{\mathrm{op}} \le \sqrt{n} \max_{i\in[n]} \|\xi_i\|_2 \le \sqrt{n} \varepsilon'\,.
\]
By the chain rule, $\nabla_H \bar{\mathcal L}(H^{(0)}) = \nabla_Z \bar{\mathcal L}(Z^{(0)}) X^\top$. Thus, the updated feature matrix is:
\begin{align}
Z^{(1)}
&=
\left(H^{(0)}-\eta \nabla_H \bar{\mathcal L}(H^{(0)})\right)X
=
Z^{(0)}-\eta \nabla_Z \bar{\mathcal L}(Z^{(0)}) X^\top X
\\
\label{eq:Z1_transient_proof}
    &= Z^{(0)} + \alpha\eta WYX^\top X -\eta \Xi X^\top X\,.
\end{align}
Using 
$
X^\top X = I_N + (B^\top B) \otimes (1_m 1_m^\top)$ and some algebra steps yields
\[
Y X^\top X = (I_K + m B^\top B) \otimes 1_m^\top\,.
\]
Finally, let the residual matrix be $R = -\Xi X^\top X$. Because the non-zero eigenvalue of $1_m 1_m^\top$ is $m$, the maximum eigenvalue of $X^\top X$ is  $1 + m \|B\|_{\mathrm{op}}^2$. Thus, using $n = Km$:
\[
\|R\|_{\mathrm{op}} \le \|\Xi\|_{\mathrm{op}} \|X^\top X\|_{\mathrm{op}} \le \sqrt{K}\sqrt{m} \varepsilon' \left( 1 + m \|B\|_{\mathrm{op}}^2 \right)\,.
\]
This completes the proof.
\end{proof}

\subsubsection{Terminal Phase}


\begin{theorem} \label{thm:gram_convergence}
Consider training with gradient descent the Spherical Bag-of-Words model \eqref{eq:nBOW}   with a fixed ETF classifier $W=I_K-\frac{1}{K}1_K1_K^\top$. Assume that gradient descent converges to a global minimizer $H_\infty$, and let $Z_\infty = H_\infty X$. 
Let $U_\infty=\lnorm{Z_\infty}$ be the corresponding matrix of normalized features. Then, the normalized Gram matrix satisfies:
$$U_\infty^\top U_\infty = \frac{K}{K-1} \left( I_K - \frac{1}{K}1_K1_K^\top \right) \otimes 1_m1_m^\top = \frac{K}{K-1} G_\text{ETF}$$
\end{theorem}
\begin{proof}

The proof follows from two main steps.

First, by \Cref{lem:ufm}, optimizing over $H$ is equivalent to optimizing over the free feature matrix $Z = HX \in \mathbb R^{K\times N}$. Because $H_\infty$ is assumed to be a global minimizer of $\bar{\mathcal{L}}$, the corresponding unnormalized feature matrix $Z_\infty$ must globally minimize the objective over all possible matrices in $\mathbb{R}^{K \times N}$. 

Note that the objective decouples over the samples as $\bar{\mathcal{L}}(Z) = \sum_{i=1}^N \ell_i(u_i)$. Because $Z_\infty$ is a free matrix, its columns $z_i$ can be varied independently. Suppose, for the sake of contradiction, that for some sample $i$, the normalized feature $u_{i,\infty}$ is not a global minimizer of $\ell_i(u_i)$ over the sphere $\mathbb{S}^{K-1}$. Then there exists some $u_i^* \in \mathbb{S}^{K-1}$ such that $\ell_i(u_i^*) < \ell_i(u_{i,\infty})$. By constructing a new matrix $\tilde{Z}$ identical to $Z_\infty$ but with the $i$-th column replaced by $u_i^*$, we would achieve a strictly lower total loss $\bar{\mathcal{L}}(\tilde{Z}) < \bar{\mathcal{L}}(Z_\infty)$, which contradicts the global optimality of $Z_\infty$.

Consequently, for every sample $i$, $u_{i,\infty}$ must strictly be a global minimizer of the per-sample spherical cross-entropy objective.
By \Cref{lem:optimal_alignment}, this minimizer is unique and satisfies:
$$u_{i,\infty} = \frac{w_{y_i}}{\|w_{y_i}\|_2} \qquad\text{for all } i\in[n].$$
Collecting the normalized columns by class, and using $\|w_c\|_2=\sqrt{\frac{K-1}{K}}$, gives the normalized feature matrix:$$U_\infty = \sqrt{\frac{K}{K-1}} W (I_K\otimes 1_m^\top).$$
Further using $W^\top W = W = I_K-\frac{1}{K}1_K1_K^\top$, we find the advertised expression for the normalized Gram matrix:
\begin{align*}U_\infty^\top U_\infty &= \frac{K}{K-1} (I_K\otimes 1_m)\left(I_K-\frac{1}{K}1_K1_K^\top\right) (I_K\otimes 1_m^\top) = \frac{K}{K-1} \left( I_K-\frac{1}{K}1_K1_K^\top \right)\otimes 1_m1_m^\top\,.
\end{align*}
\end{proof}

\begin{remark}[Global Convergence Assumption]
The assumption that gradient flow converges to a global minimizer is a common simplifying assumption in the UFM and Neural Collapse literature \cite{ fang2021exploring,zhu2021geometric,thrampoulidis2022imbalance}. Since the objective \(\bar{\mathcal{L}}\) is non-convex, establishing global convergence from random initialization is in general difficult. This issue remains in our setting: even with \(W\) fixed, the optimization over \(H\) is still non-convex because of the normalization. Accordingly, we adopt this assumption to separate the analysis of the terminal representation from the analysis of the optimization dynamics, and focus here on the geometry of the limiting features.
\end{remark}

\subsubsection*{Supporting Lemmas}

\begin{lemma}[Equivalence to the Unconstrained Feature Model] \label{lem:ufm}
Optimizing the objective $\bar{\mathcal{L}}$ over $H \in \mathbb{R}^{K \times V}$ is equivalent to optimizing over the free variables $Z \in \mathbb{R}^{K \times N}$.
\end{lemma}

\begin{proof} It suffices to show that for arbitrary matrix $Z \in \mathbb{R}^{K \times N}$  there exists a weight matrix $H$ such that $HX = Z$. This is possible thanks to the structure of the data matrix $X = \begin{bmatrix} I_N \\ B \otimes 1_m^\top \end{bmatrix}$. Indeed, the desired is true by setting $H=\begin{bmatrix} Z & 0_{K \times s} \end{bmatrix}$. 
\end{proof}



\begin{lemma}[Optimal Directional Alignment] \label{lem:optimal_alignment}
Suppose $W=I_K-\frac{1}{K}1_K1_K^\top$ is fixed to an ETF. For any sample $i$ with  class label $y_i \in [K]$, the minimization of the cross-entropy loss over the sphere
$$ \min_{u_i \in \mathbb{S}^{K-1}} \left\{\ell_i(u_i) = -\log\left( \sfti{y_i}{Wu_i} \right) \right\}\,, $$
has a unique global minimizer given by $u_i^* = \frac{w_{y_i}}{\|w_{y_i}\|_2}$, where $w_{c}, c\in[K]$ denotes the c-th row of $W$.
\end{lemma}

\begin{proof}
Denote $c = w_{y_i}^\top u_i$. Because $u_i \in \mathbb{S}^{K-1}$, by Cauchy-Schwarz inequality dictates that:
$$ c \le \|w_{y_i}\|_2 \,, $$
with equality holding if and only if $u_i$ is perfectly aligned with $w_{y_i}$, i.e., $u_i = \frac{w_{y_i}}{\|w_{y_i}\|_2}$.

Moreover, since $\sum_{c=1}^K w_c = {0}$:
$$ \sum_{c \neq y_i} w_c^\top u_i = - w_{y_i}^\top u_i = -c \,. $$
From this, and Jensen's inequality applied to the strictly convex exponential function:
$$ \frac{1}{K-1} \sum_{j \neq y_i} \exp(w_j^\top u_i) \ge \exp\left( \frac{1}{K-1} \sum_{j \neq y_i} w_j^\top u_i \right) = \exp\left( \frac{-c}{K-1} \right) \,, $$
with equality if and only if $w_c^\top u_i$ is identical for all $c \neq y_i$. 

Substituting this lower bound back into the loss function yields a lower bound on $\ell(u_i)$ as a function of $c$:
$$ \ell(u_i) \ge \log \left( 1 + (K-1) \exp\left( \frac{-c}{K-1} - c \right) \right) = \log \left( 1 + (K-1) \exp\left( \frac{-K}{K-1} c \right) \right) \,. $$
The lower bound is strictly  decreasing function of $c$. Therefore, to globally minimize the loss, we must strictly maximize $c$. 

But  $c$ attains its  maximum uniquely at $u_i^* = \frac{w_{y_i}}{\|w_{y_i}\|_2}$. At this point it  holds that
$$ w_c^\top u_i^* = \frac{w_c^\top w_{y_i}}{\|w_{y_i}\|_2} = \frac{-1/K}{\sqrt{1-1/K}} = \frac{-1}{\sqrt{K(K-1)}} \,. $$
Thus, the equality condition for Jensen's inequality is satisfied. We conclude that the lower bound is achieved and $u_i^* = \frac{w_{y_i}}{\|w_{y_i}\|_2}$ is the unique global minimizer.
\end{proof}

\subsection{Direct simulation of the Spherical BoW dynamics}\label{sec:theory_sims}
To complement the theoretical analysis, we run simulation with Spherical BoW model and Color-Shape data with semantic basis matrix $B$ as in Eq. \eqref{eq:B matrix}. We use \(K=8\) label classes, \(m=5\) samples per class, so \(n=Km=40\) total samples, and \(s=6\) shared semantic features corresponding to the two colors and four shapes. This input matrix is: 
\[
X=\begin{bmatrix} I_n \\ B \otimes {1}_m^\top \end{bmatrix}\in\mathbb{R}^{V\times n},
\qquad V=n+s=46,
\]
where \(B\in\mathbb{R}^{6\times 8}\) is the semantic base matrix. We sweep hidden dimension \(d\in\{8,16,32,64\}\), initialize all trainable parameters with small zero-mean Gaussians of standard deviation \(0.05\), and run full-batch gradient descent for \(10^5\) steps with learning rate \(10^{-3}\).
We consider two settings. In the first, the classifier is fixed to an ETF readout and only \(H\) is trained (aligned with theoretical analysis of Sec. \ref{sec:theorems}). In the second, both \(W\) and \(H\) are trained jointly. At each step we compute the empirical RSM
\[
G_{\mathrm{emp}} = X^\top H^\top H X
\]
and compute its Pearson correlation with semantic and ETF hypothesis RSMs. 
The results in \Cref{fig:theory_numerical} show the same two-stage behavior as in our Transformer experiments. In both settings, \(G_{\mathrm{emp}}\) first moves toward the semantic geometry: the early empirical RSM exhibits the same block structure as Sem-RSM, and correlation with Sem-RSM rises sharply. Later in training, this semantic alignment declines while correlation with ETF-RSM increases, and the final empirical RSM approaches the label-driven ETF structure.

Two points are notable. First, the transient semantic phase already appears when \(W\) is fixed to ETF, which is the setting analyzed in our theory. This shows that a trainable classifier is not necessary for the emergence of the semantic phase. Second, increasing \(d\) makes both phases more pronounced: larger hidden dimensions reach higher early alignment with Sem-RSM and later converge more fully toward ETF-RSM. Allowing \(W\) to be trainable mainly shifts the transition to later times, but does not change the qualitative picture.

\begin{figure}
    \centering
    \makebox[\linewidth][c]{
    \includegraphics[width=1.2\linewidth]{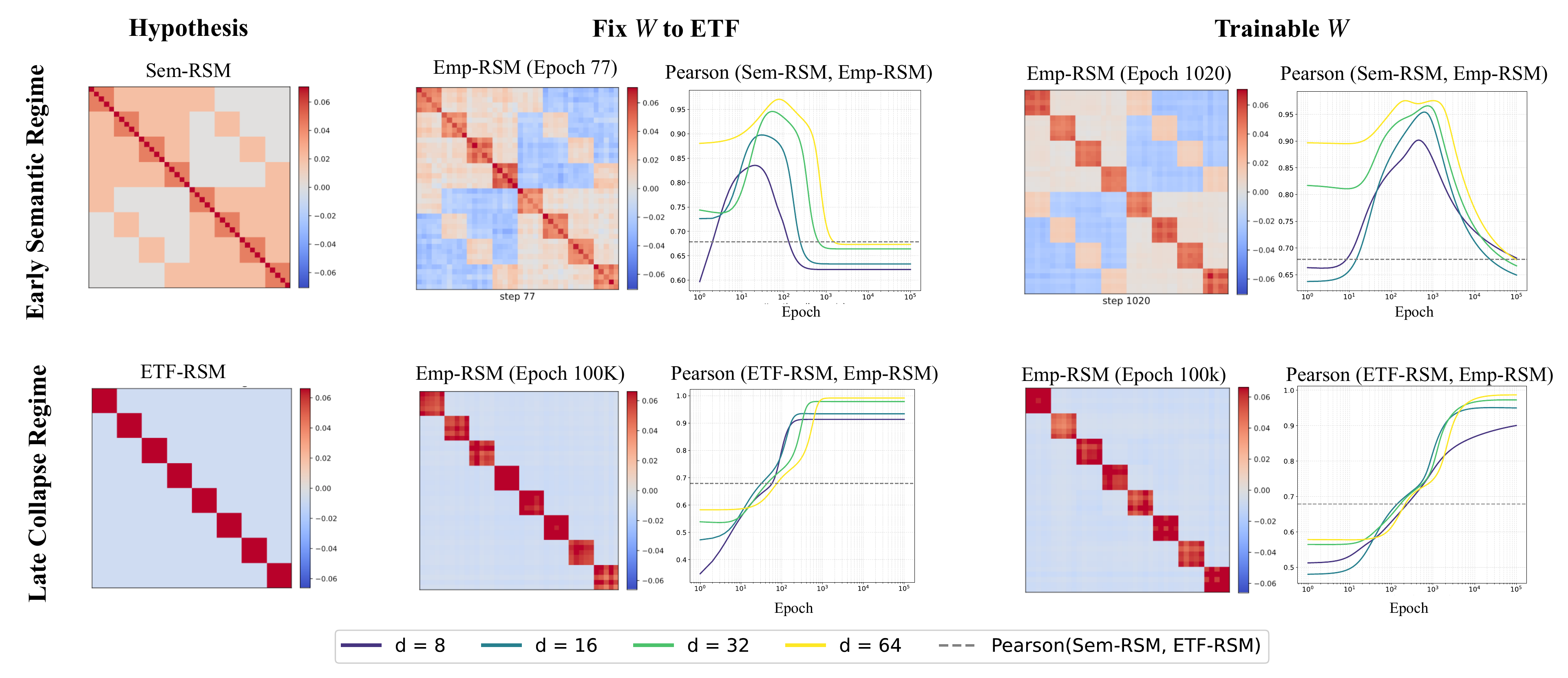}
    }
    \caption{\textbf{Numerical simulation of the Spherical BoW dynamics on the color-shape setup.}
\textbf{Left column:} hypothesis RSMs. Sem-RSM encodes the semantic similarity structure, while ETF-RSM encodes the label-driven collapsed geometry.
\textbf{Middle and Right columns:} Simulations of the Spherical BoW objective under two settings: fixing \(W\) to an ETF readout, and training \(W\) jointly with \(H\). In each setting, we show a representative empirical RSM in the early semantic regime and at the end of training, together with the Pearson correlation of \(G_{\mathrm{emp}}=X^\top H^\top H X\) with Sem-RSM and ETF-RSM over training, for \(d\in\{8,16,32,64\}\). In each setting, the early empirical RSM is taken from the epoch at which the correlation with Sem-RSM is maximal, and the late empirical RSM is taken from the final training epoch.
In both settings, the empirical geometry first aligns with the semantic structure and later approaches the ETF geometry. Larger \(d\) makes both the early semantic peak and the late ETF convergence more pronounced. The same transient semantic phase appears even when \(W\) is fixed to ETF, showing that joint training of the classifier is not necessary for the effect.}
    \label{fig:theory_numerical}
\end{figure}

\subsection{Spherical BoW on the Color-Shape Categories language}\label{sec:theory_sims_2}
\begin{figure}
    \centering
    \makebox[\linewidth][c]{
    \includegraphics[width=1.2\linewidth]{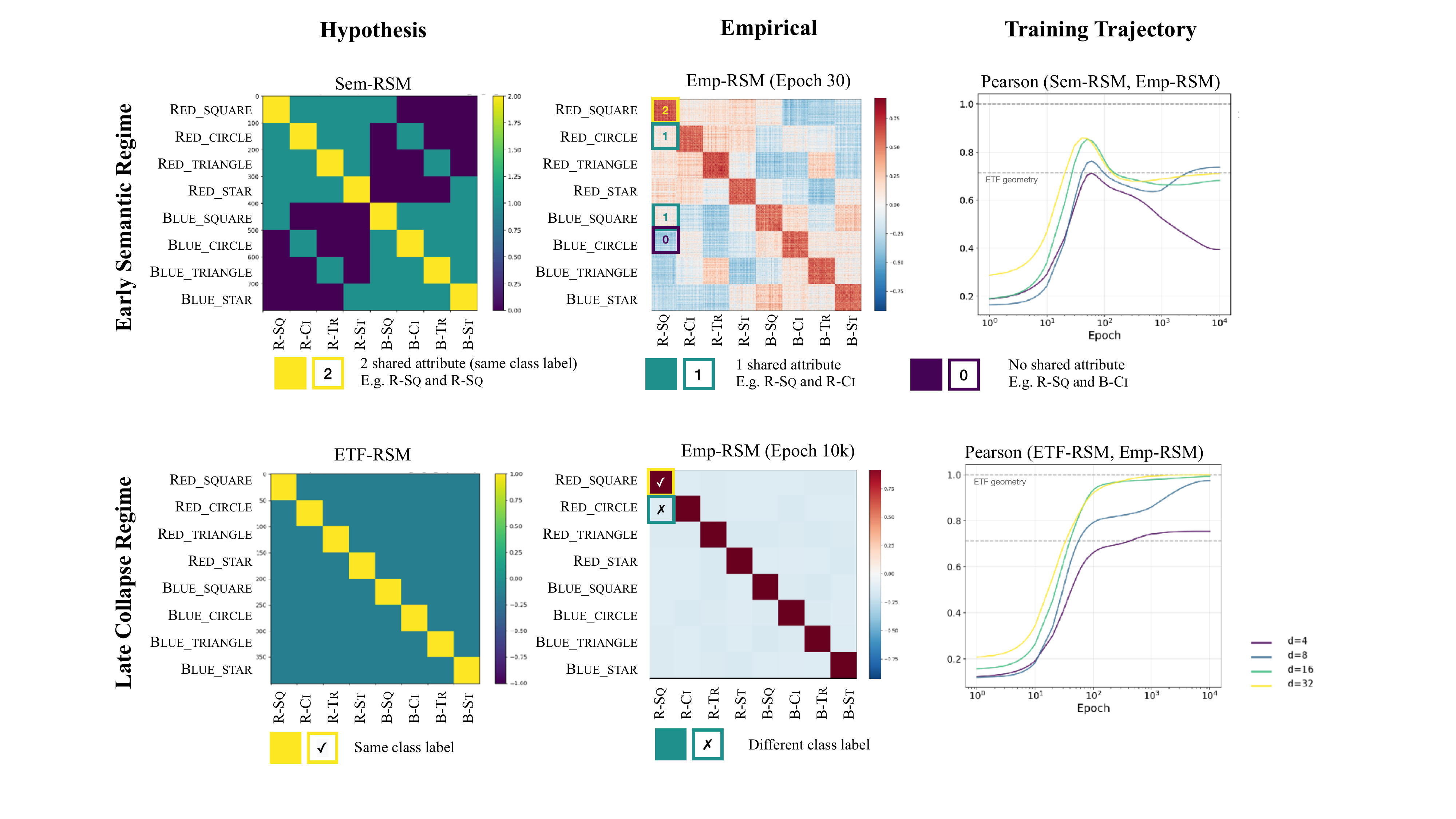}
    }
    \caption{\textbf{Spherical BoW on the full Color-Shape Categories language.}
This figure is the Spherical BoW analogue of \cref{fig:main_results}. \textbf{Left column:} hypothesis RSMs from Experiment~1. Sem-RSM encodes semantic overlap between classes, with values 2, 1, and 0 for pairs sharing two, one, or no latent attributes, while ETF-RSM encodes the label-driven collapsed geometry. \textbf{Middle column:} empirical RSMs at an early epoch (30) and a late epoch (10k). Early in training, the empirical geometry shows the graded semantic block structure of Sem-RSM. By the end of training, this structure weakens and the empirical RSM becomes close to the ETF geometry. \textbf{Right column:} Pearson correlation between Emp-RSM and each hypothesis RSM over training, for \(d \in \{4,8,16,32\}\). As in the Transformer experiments, alignment with Sem-RSM rises early and then declines, while alignment with ETF-RSM increases later. Larger hidden dimensions show both stronger early semantic alignment and more complete late convergence to the label-driven geometry.}
    \label{fig:cufm_language1_result}
\end{figure}
We next test the Spherical BoW model on the full Color-Shape Categories language from Experiment~1, described in \cref{sec:exp1}, rather than on the simplified analytic design matrix used in the direct simulation above. We use the same dataset as in the main experiment: \(n=800\) samples across \(K=8\) label classes, with one nuisance ID per sample and contextual tokens drawn from fine-grained color and shape variants. Thus, the training labels remain strictly one-hot, while the inputs still contain the latent semantic overlap structure of Experiment~1.

We train the Spherical BoW model on multi-hot encodings of these input contexts and track the empirical representational similarity matrix
throughout training. We sweep hidden dimension \(d \in \{4,8,16,32\}\), and compare the resulting empirical RSM to the same two hypothesis RSMs used in the main experiment: Sem-RSM, which reflects semantic overlap between classes, and ETF-RSM, which reflects the label-driven collapsed geometry.

The results in \Cref{fig:cufm_language1_result} match the Transformer dynamics from Experiment~1. For models with unconstrained dimension ($d>K$), early in training, the empirical RSM exhibits the block structure of Sem-RSM, and correlation with Sem-RSM rises. Later in training, this semantic structure weakens, correlation with Sem-RSM declines, and correlation with ETF-RSM increases monotonically toward the label-driven geometry.
As in Experiment~1, the hidden dimension controls how strongly both phases appear. Larger \(d\) leads to stronger early semantic alignment and more complete late convergence to ETF-RSM, while the smallest model (\(d=4\)) shows both a weaker semantic peak and a less complete late collapse. This again supports the view that the semantic phase is not caused by a representational bottleneck. Rather, even in this much simpler architecture, larger models recover the semantic geometry more strongly before later training shifts them toward the label-driven regime.

\subsection{Limitations}
We view our Spherical CBoW as a first step toward establishing analytical proxies for studying semantic-driven geometries, analogous to the role the UFM plays for label-driven NC geometries. In the NC literature, the UFM is valuable not only because it explains already-observed geometries, but even more so because it forms the basis for \textit{predicting} geometry in new settings (e.g.\ imbalanced data or soft labels). An analogous framework for semantic geometry would be highly desirable, but the challenge is substantial. The UFM captures the large-model regime by treating embeddings as free optimization variables — an assumption that is intuitive for sufficiently expressive models and has been rigorized in certain settings \citep{jacot2024wide,sukenik2025neural,garrod2024persistence}. However, this very assumption erases input semantic structure by design. At the same time, we have seen that architectural bottlenecking alone cannot explain the strong semantic geometry that emerges in larger models, which fully reflects the input data structure. Our Spherical BoW preserves this structure and reveals the transient semantic phase, but it lacks the intuitive expressiveness of the UFM. This is perhaps to be expected: for instance, its bag-of-words assumption yields an order-independent model that cannot distinguish two input contexts differing only in the positional arrangement of their semantic tokens. Extending the Spherical BoW framework to address these limitations is an exciting direction for future work, and we view both this model and our experiments as a starting point.

\end{document}